\documentclass{article}

\usepackage[numbers]{natbib}

\usepackage[final]{neurips_2022}

\usepackage{microtype}
\usepackage{graphicx}
\usepackage{subfig}

\usepackage[utf8]{inputenc} 
\usepackage[T1]{fontenc}    
\usepackage[pdftex,colorlinks=true,citecolor=blue]{hyperref}
\usepackage{url}            
\usepackage{booktabs}       
\usepackage{amsfonts}       
\usepackage{nicefrac}       
\usepackage{microtype}      
\usepackage{xcolor}         
\usepackage{algorithm}
\usepackage{algcompatible}
\usepackage{algpseudocode}
\usepackage{amsmath}
\usepackage{amssymb}
\usepackage{mathtools}
\usepackage{amsthm}
\usepackage{wrapfig}

\usepackage{natbib}

\usepackage[capitalize,noabbrev]{cleveref}

\theoremstyle{plain}
\newtheorem{theorem}{Theorem}[section]

\newtheorem{lemma}[theorem]{Lemma}

\theoremstyle{definition}
\newtheorem{definition}[]{Definition}

\theoremstyle{remark}

\usepackage[textsize=tiny]{todonotes}

\usepackage{xspace}
\usepackage{caption}
\usepackage{microtype}
\usepackage{graphicx}
  \graphicspath{{./figs/}}
  \DeclareGraphicsExtensions{.pdf,.jpg,.png,.eps}
\usepackage{subfig}
\usepackage{booktabs} 
\usepackage{makecell}
\usepackage{wrapfig,lipsum,booktabs}

\newcommand{\secref}[1]{Section~\ref{#1}}
\newcommand{\figref}[1]{Figure~\ref{#1}}
\newcommand{\subfig}[1]{\textit{#1}}
\newcommand{\tabref}[1]{Table~\ref{#1}}
\newcommand{\algoref}[1]{Algorithm~\ref{#1}}

\newcommand{\ie}{\textrm{i.e.}\xspace}
\newcommand{\eg}{\textrm{e.g.}\xspace}

\usepackage{amsmath,amsfonts,amssymb}
\usepackage{amsfonts}
\usepackage{comment}
\usepackage{multirow}

\newcommand{\algo}{RHIRL\xspace}
\newcommand{\actualControl}[1]{v_{#1}}

\newcommand{\dkl}{D_{\sss \textrm{KL}}}
\newcommand{\state}[1]{x_{#1}}
\newcommand{\scost}{\ensuremath{g}\xspace}
\newcommand{\inputControl}[1]{u_{#1}}

\newcommand{\noiseMatrix}{\Sigma}

\newcommand{\stateCostFunc}{S}
\newcommand{\controlSeqCostFunc}{J}

\newcommand{\optimalActControlSeq}{V^*}

\newcommand{\partition}{Z}
\newcommand{\traj}{\tau}

\newcommand{\sss}{\scriptscriptstyle}
\newcommand{\inverseTemp}{\lambda}

\newcommand{\baseControlInput}{\ensuremath{U_{\sss B}}\xspace}

\newcommand{\objective}[2]{{\mathcal{L}}(#2;#1)}
\newcommand{\obj}{\ensuremath{\mathcal{L}}}
\newcommand{\expertDemo}{\mathcal{D}}
\newcommand{\pinit}{\ensuremath{\mu}\xspace}
\newcommand{\stateinit}{x_0}

\newcommand{\thetaE}{\theta_{\sss E}}
\newcommand{\expertStateDistribution}{p_{\sss \textrm{\upshape E}}}
\newcommand{\RHCStateDistribution}{p_{\sss \textrm{\upshape RHC}}}

\makeatletter

\newcommand*{\@rowstyle}{}

\newcommand*{\rowstyle}[1]{
  \gdef\@rowstyle{#1}%
  \@rowstyle\ignorespaces%
}
\usepackage{amsthm}
\makeatletter
\newtheorem*{rep@theorem}{\rep@title}
\newcommand{\newreptheorem}[2]{%
\newenvironment{rep#1}[1]{%
 \def\rep@title{#2 \ref{##1}}%
 \begin{rep@theorem}}%
 {\end{rep@theorem}}}
\makeatother

\title{Receding Horizon Inverse Reinforcement Learning}

\author{
Yiqing Xu$^{1}$
  \qquad
  Wei Gao$^{1}$
  \qquad
  David Hsu$^{1,2}$  \vspace{0.25cm}\\
  $^1$School of Computing  \\
  $^2$Smart Systems Insitute\\
  National University of Singapore\\
  \texttt{\{xuyiqing,gaowei90,dyhsu\}@comp.nus.edu.sg} \\
}

\begin{document}
\maketitle

\begin{abstract}
  Inverse reinforcement learning (IRL) seeks to infer a cost function that explains the underlying goals and  preferences of  expert demonstrations.
This paper presents \textit{receding-horizon inverse reinforcement learning} (\algo), an IRL algorithm for high-dimensional, noisy, continuous systems with black-box dynamic models. \algo addresses two key challenges of IRL: scalability and robustness. To handle high-dimensional continuous systems, \algo  matches the induced optimal trajectories with expert demonstrations \textit{locally} in a receding horizon manner and ``stitches'' together the local solutions to learn the cost; it thereby avoids the ``curse of dimensionality''.
This contrasts with  earlier algorithms that match with expert demonstrations \textit{globally} over the entire high-dimensional state space. To be robust against imperfect expert demonstrations and control noise, \algo learns a state-dependent cost function ``disentangled'' from system dynamics under mild conditions. Experiments on benchmark tasks show that  \algo outperforms several leading IRL algorithms in most instances. 
We also prove that the cumulative error of \algo grows linearly with the task duration. 
\end{abstract}

\section{Introduction}

\label{introduction}
Reinforcement learning (RL) has made exciting progress in a range of complex tasks, including real-time game playing \citep{DBLP:journals/corr/MnihKSGAWR13},  visuo-motor control of robots~\citep{viereck2017learning}, and many other works. The success, however,  often hinges on a carefully crafted cost function \citep{KnoAll21,NgHar99}, which is a major impediment to the wide adoption of RL in  practice. Inverse reinforcement learning (IRL) \citep{Ng00algorithmsfor} addresses this need by learning a cost function that explains the underlying goals and preferences of expert demonstrations. This work focuses on two key challenges in IRL, \textit{scalability} and \textit{robustness}.

Classic IRL algorithms commonly consist of two nested loops. The inner loop
approximates the optimal control policy for a hypothesized cost function,
while the outer loop updates the cost function by comparing the behavior of the induced policy
with expert demonstrations. The inner loop must solve the (forward) reinforcement
learning or optimal control problem, which is in itself a challenge for
complex high-dimensional systems. Many interesting ideas have been proposed
for IRL, including, \eg, maximum entropy learning~\citep{DeepMaxEnt,MaxEntIRL},
guided cost learning~\citep{GCL}, and adversarial
learning~\citep{AIRL}. See~\figref{fig:comparison} for illustrations. They try
to match a \textit{globally} optimal approximate policy with expert
demonstrations over the entire system state space or a sampled approximation
of it.  This is impractical for high-dimensional continuous systems and is a
fundamental impediment to scalability.
Like RL, IRL suffers from the same ``curse of dimensionality''.
To scale up, \emph{receding-horizon IRL}  (\algo) computes  locally optimal policies with  receding horizons
rather than a globally optimal policy and then matches them with expert
demonstrations \textit{locally} in  succession (\figref{fig:comparison}\subfig
d).  The local approximation and matching substantially mitigate the impact of
high-dimensional space and improve the sample efficiency of \algo, at the cost
of a local rather than a global solution. So \algo trades off optimality for
scalability and provides an alternative to current  approaches.

Another important concern of IRL is noise in expert demonstrations and system
control. Human experts may be imperfect for various reasons and provide good,
but still suboptimal demonstrations. Further, the system may fail to execute
the commanded actions accurately because of control noise. We want to learn a
cost function that captures the expert's intended actions rather than the
imperfectly executed actions.  While earlier work has investigated the
question of  learning from sub-optimal or failed demonstrations \citep{brown2019extrapolating, failureInverseRL,
  wu2019imitation}, there is a subtle, but critical difference between
(i) demonstrations intended to be suboptimal and (ii) optimal demonstrations
corrupted by noise. The existing work \citep{brown2019extrapolating, failureInverseRL,
  wu2019imitation}  addresses (i); RHIRL
addresses (ii). To learn the true intentions from the noise corrupted
demonstrations, \algo relies on a simplifying assumption: the cost function is
linearly separable with two components, one state-dependent and one
control-dependent. Many interesting systems in practice satisfy the
assumption, at least, approximately \citep{Pow07}. \algo then learns the
state-dependent cost component, which is disentangled from the system dynamics
\citep{AIRL} and agnostic to noise.

\section{Related Work}
\label{sec:related}

\begin{figure}
\vspace{-8px}
\captionsetup[subfigure]{font=scriptsize,labelfont=scriptsize,position=bottom,labelformat=empty}
\centering
\subfloat[(\subfig a)
\label{fig:comparison_a}]{\includegraphics[width=0.25\columnwidth]{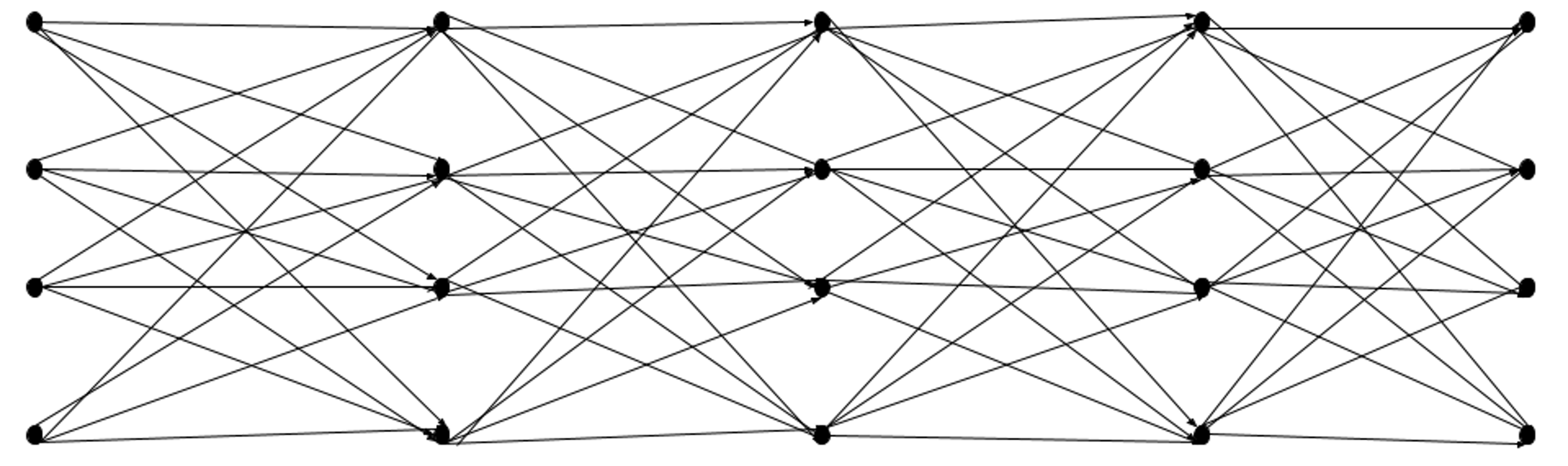}}
\subfloat[(\subfig b)\label{fig:comparison_b}]{\includegraphics[width=0.25\columnwidth]{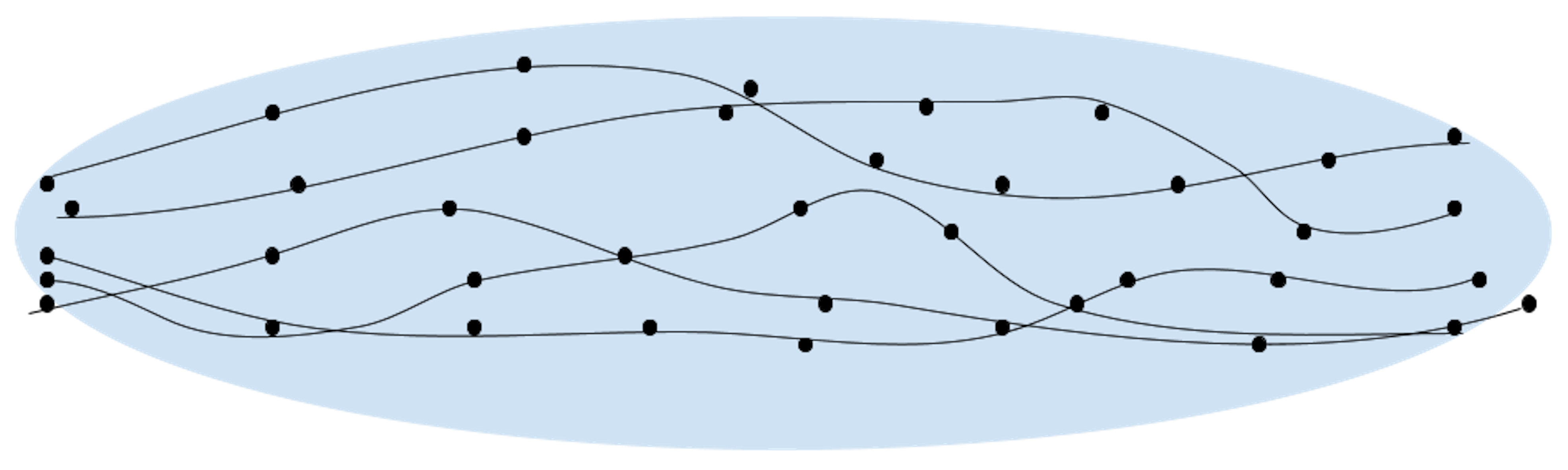}}
\subfloat[(\subfig c)\label{fig:comparison_c}]{\includegraphics[width=0.25\columnwidth]{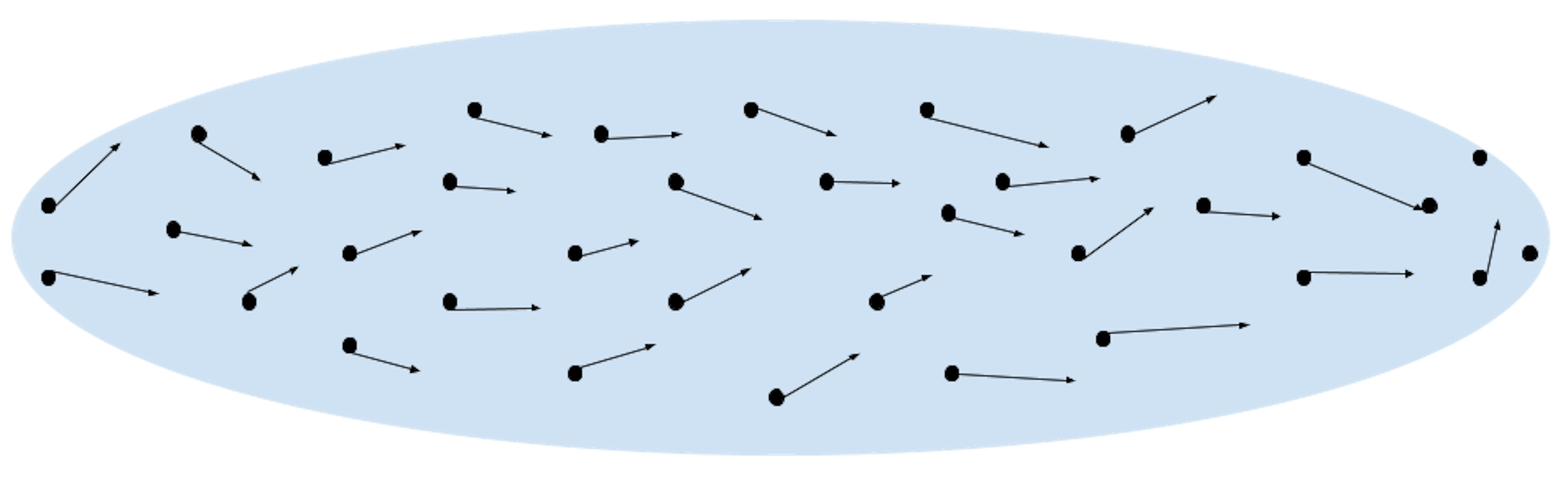}}
\subfloat[(\subfig d)\label{fig:comparison_d}]{\includegraphics[width=0.25\columnwidth]{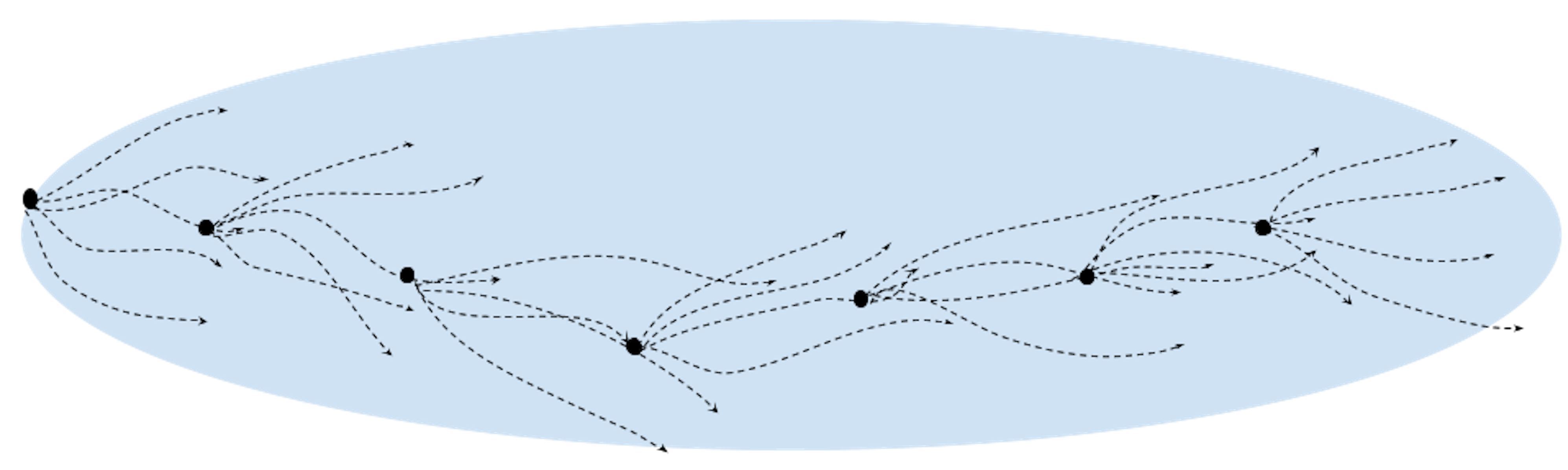}}
\vspace{-5px}
\caption{
A comparison of \algo and  selected IRL algorithms. They all try to match the policy induced by the learned cost  with  expert demonstrations.  (\subfig a) {MaxEnt}  matches the exact feature count over the entire system state space. (\subfig b) REIRL, GCL, and GAN-GCL match the approximate feature count over sampled state trajectories globally over the entire task duration.
(\subfig c)~GAIL, AIRL, ... discriminate the sampled state or state-action distributions.
(\subfig d) \algo 
matches control sequences \textit{locally} along demonstrated trajectories in a receding horizon manner.
}
\label{fig:comparison}
\vspace{-10px}
\end{figure}
IRL can be viewed as an indirect approach to imitation learning. It learns a
cost function, which induces an optimal policy whose behavior matches with expert
demonstrations. In contrast, behavior cloning (BC) is a direct approach.  It
assumes independence among all demonstrated state-action pairs and learns a
policy that maps states to actions through supervised learning of state-action
pairs from expert demonstrations.  The simplicity of BC is appealing. However,
it typically requires large amounts of data to learn well and suffers from
covariate shift~\citep{ross2011reduction}.  IRL is more data-efficient.
Further, it produces a cost function, which explains the expert demonstrations and potentially  transfers to other systems with different dynamics. These benefits, however, come at the expense of greater computational complexity. 

Classic IRL algorithms  learn a cost function iteratively in a double-loop: the  outer loop updates a hypothesized cost function, and  the inner loop solves the forward RL problem for an optimal policy and matches it with expert demonstrations. Various methods have been proposed  \citep{REIRL,GCL,DeepMaxEnt,MaxEntIRL}, but they all seek a globally optimal solution  over the entire state space (\figref{fig:comparison}\subfig a) \citep{DeepMaxEnt,MaxEntIRL} or the entire task duration (\figref{fig:comparison}\subfig b) \citep{REIRL,GCL}. As a result, they face significant challenges in scalability and must make simplifications, such as locally linear dynamics~\citep{GCL}. 
 Recent methods use 
the generative adversarial network (GAN)~\citep{GAN} to learn a discriminator that differentiates between  the state or state-action distribution  induced by the learned cost and that from expert demonstrations~\citep{SSRR,AIRL,FAIRL,GAIL,OPIRL,DAC,SMM,fIRL,EAIRL}. 
We view this global matching as a major obstacle to scalability. In addition, GAN training is challenging and faces difficulties with convergence.

\algo stands in between BC and the aforementioned IRL algorithms by trading
off optimality for scalability. BC performs local matching at each
demonstrated state-action pair, treating all of them independently. Most
existing IRL algorithms perform global matching over the entire state space or
sampled trajectories from it.
\algo follows the standard
IRL setup.  To tackle the challenge of scalability for high-dimensional
continuous systems, \algo borrows ideas from \emph{receding horizon
control} (RHC) \citep{RHL}. It solves for locally optimal control sequences with
receding horizons and learns the cost function by ``stitching'' together a
series of locally optimal solutions to match the global state distribution of
expert demonstrations (\figref{fig:comparison}\subfig d). Earlier work
has explored the idea of RHC in IRL \citep{BetweenImitationAndIntention}, but 
 relies on handcrafted reward features and is limited to   discrete, low-dimensional tasks.
More recent work learns reward features automatically \citep{CostMPC}. However,
it focuses on lane navigation for autonomous driving, exploting a known
analytic model of dynamics and noise-free, perfect expert demonstrations.
\algo  aims at general high-dimensional continuous tasks with noisey expert
demonstrations.

  Another challenge to IRL is  suboptimal expert demonstrations and system
  control noise. Several methods learn an auxiliary score or ranking  to
  reweigh the demonstrations, in order to approximate the underlying optimal
  expert distribution
  \citep{brown2019betterthandemonstrator,SSRR,wu2019imitation}. \algo does not
  attempt to reconstruct the optimal expert demonstrations. It explicitly
  models  the human errors in control actions as additive Gaussian noise and
  matches the noisy control with expert demonstrations, in order to learn from
  the intended, rather than the executed expert actions. Modelling the human
  error as additive Gaussian noise is a natural choice technically \citep{pmlr-v28-levine13,ofc,
    MI_via_MCE}, but compromises on realism. Human errors
  in sequential decision making may admit other structural regularities, as a
  result of planning, resource constraints, uncertainty, or bounded
  rationality\citep{10.2307/1914185}. Studying the specific forms of human
  errors in the IRL context  requires insights beyond the scope of our current
  work, but forms a promising direction for future investigation.

\section{Receding Horizon Inverse Reinforcement Learning}
\label{rhirl}
\subsection{Overview}
Consider a continuous dynamical system:
\begin{equation}
  \label{trans_func}
  x_{t+1} = f(x_t, v_t),
\end{equation}
where $\state{t} \in \mathcal{R}^{n}$ is the state,  $\actualControl{t} \in \mathcal{R}^{m}$ is the control  at time $t$, and the initial system state at $t=0$ follows a distribution \pinit.
To account for noise in expert demonstrations, we assume that $\actualControl{t}$ is a random variable following the Gaussian distribution $ \mathcal{N} ( v_t | \inputControl{t}, \Sigma)$, with mean $\inputControl{t}$ and covariance $\Sigma$.
 We can control  $\inputControl{t}$ directly, but not~$\actualControl{t}$, because of noise. The state-transition function $f$ captures the system dynamics. \algo represents $f$ as a black-box simulator and does not require its analytic form. Thus, we can accommodate arbitrary complex nonlinear dynamics. To simplify the presentation, we assume that the system dynamics is deterministic. We sketch the extension to stochastic dynamics at the end of the section, the full proof is given in Appendix \ref{Appendix:extension_stochastic}.

In RL, we are given a cost function  and want to find a control policy that minimizes the expected total cost over time under the dynamical system. In IRL, we are not given the cost function, but instead, a set $\expertDemo$ of expert demonstrations. Each demonstration is a  trajectory of states visited by the expert over time:  $\traj = (x_0, x_1, x_2, ... , x_{\sss {T}-1})$ for a  duration of  $T$ steps.

We assume that the expert chooses the actions to minimize an unknown cost function and want to recover this cost  from the demonstrations. Formally, suppose that the cost function is parameterized by $\theta$. \algo aims to learn a cost function that minimizes the loss $\objective{\expertDemo} {\theta}$, which measures
the difference between  the demonstration trajectories and the optimal control policy induced by the cost function with parameters~$\theta$.

\algo performs this minimization iteratively, using the gradient
$\partial \obj / \partial\theta $ to update $\theta$.  In iteration~$t$, let
$x_t$ be the system state at time $t$ and $\expertDemo_t$ be the set of expert
sub-trajectories starting at time $t$ and having a duration of maximum $K$
steps. We use the current cost to perform receding horizon control (RHC) at
$x_t$, with time horizon $K$, and then update the cost by comparing the
resulting state trajectory distribution with the demonstrations in
$\expertDemo_t$. Earlier IRL work 
requires the simulator to reset to the states from the initial state
distribution \citep{REIRL, GAN-GCL, GCL,  AIRL, GAIL, DAC}.
To perform receding horizon optimization, \algo requires a
simulator to reset to the initial state $x_t$ for the local control sequence
optimization at each time step $t = 0, … T-1$, a common setting in
optimal control (\eg, \citep{Williams2018RobustSB,MPPI,CostMPC}).
See \algoref{algo} for a sketch.
 
\begin{algorithm}
\caption{\algo}\label{algo}
\small
\begin{algorithmic}[1]
\State Initialize $\theta$ randomly.
\For {\texttt{$i= 1,2,3, ...$}}
\State Sample $\stateinit$ from \pinit. 
\State Initialize control sequence $U$ of length $K$ to $(0,0,\dots)$. 
\For{\texttt{$t = 0,1,2, ..., T-1$}}
\State Sample $M$ control sequences $V_j$ for $j=1,2,\ldots M$, according to $ \mathcal{N}(V | U,\Sigma)$. \label{a:sample}
\State Compute the importance weight $w_j$ for each $V_j$, using the state cost $S(V_j, x_t; \theta)$.~\Comment{See equation~\eqref{eqn:weights}}.\label{a:grad1}
\State Compute $ {\partial \over \partial \theta } \obj(\theta; \expertDemo_t, x_t)$ using $\expertDemo_t$ and $V_j$ with weight  $w_j$, for  $j=1, 2, \ldots, M$.~\Comment{See equation~\eqref{equ:grad}}.  \label{a:grad2}
\State $\theta \gets \theta -\alpha\, {\partial\obj \over \partial \theta}$. \label{a:update}  
\State $U \gets \sum_{j=1}^{M} w_j V_j$.\label{a:stateUpdate1}
\State Sample $v_t$  from $ \mathcal{N}(v| u,\Sigma )$, where $u$ is the first element in the control sequence $U$. 
\State $x_{t+1} \gets f(x_t, v_t)$. \label{a:stateUpdate2}
\State Remove $u$ from $U$. Append $0$  at the end of $U$. \label{a:controlUpdate}
\EndFor
\EndFor
\end{algorithmic}\label{algo}
\end{algorithm}

  First,  sample $M$ control sequences, each of length~$K$ (line~6).
  We assume that the covariance $\noiseMatrix$ is known. If it is unknown, we set it to be identity by default. For our experiment results reported in \tabref{tab:scalablity} and \tabref{tab:noise}, $\noiseMatrix$ is unknown and is approximated by a constant factor of the identity matrix (grid search is performed to determine this constant factor). We show through experiments that the learned cost function is robust over  different noise settings (\secref{sec:robustness}). 
  Next, we apply model-predictive path integral (MPPI) control~\citep{MPPI} at $x_t$. MPPI provides an analytic solution for the optimal control sequence distribution and the associated state sequence distribution, which allow us to estimate the gradient $\partial \obj / \partial\theta $ efficiently through importance sampling (lines~7--8) and update cost function parameters $\theta$ (line~9). Finally, we execute the computed optimal control (line 10--12) and update the mean control input for the next iteration (line~13).
  We would like to emphasize that we only use the simulator to sample rollouts and evaluate the current cost function. Fixed expert demonstrations $\expertDemo$ are given upfront. Unlike DAgger, we do not query the expert online during training. 

Our challenge is to uncover a cost function that captures the expert's intended controls $(\inputControl{0}, \inputControl{1}, \ldots, \inputControl{{\sss T}-1})$, even though they were not directly observed,  because of noise, and do so in a scalable and robust manner for high-dimensional, noisy systems. 

We develop three ideas: structuring the cost function, matching locally with expert demonstrations, and efficient computation of the gradient $\partial \obj / \partial\theta $, which are described next.

\subsection{Robust Cost Functions}
\label{sec:cost}
To learn a cost function robust against noise, we make a simplifying assumption that linearly separates the one-step cost into two components: a state cost  $\scost(x; \theta)$ parameterized by $\theta$ and a quadratic control cost $u^{T} \noiseMatrix^{-1} u$.  Despite the simplification, this cost function models a wide variety of interesting systems in practice~\citep{Pow07}. It allows \algo to learn a state cost $\scost(x; \theta)$, independent of  control noise (\secref{sec:grad}), and thus generalize over different noise distributions (\secref{sec:robustness}).

Suppose that $V=(v_0, v_1, \ldots, v_{{\sss K}-1})$ is a control sequence of length $K$,  conditioned on the input $U=(u_0, u_1, \ldots, u_{{\sss K}-1})$.
 We  apply $V$ to the dynamical system in \eqref{trans_func} with start state  $\stateinit$ and obtain a state sequence $\traj=(\stateinit, x_1, \ldots, x_{\sss K})$  with $x_k = f(x_{k-1}, v_{k-1})$ for $k=1,2, \ldots, K$. 
Define the total cost of $V$ as 
\begin{equation}\label{equ:cost}
    \controlSeqCostFunc(V, \stateinit; \theta) =\sum_{k = 0}^{K} \scost(x_k;\theta) + \sum_{k=0}^{K-1}\frac{\inverseTemp}{2} 
    \inputControl{k}^{\intercal}
    \noiseMatrix^{-1} \inputControl{k}, 
\end{equation}
where $\inverseTemp \geq 0$ is a constant weighting the relative importance between the state  and  control costs. For convenience, define also the total state cost of $V$ as
\begin{equation}\label{equ:statecost}
\stateCostFunc(V, \stateinit;\theta) = \sum_{k = 0}^{K} \scost(x_k;\theta).
\end{equation}
While $\stateCostFunc$ is defined in terms of the control sequence $V$, it only depends on the corresponding state trajectory~$\traj$. This is very useful, as the training data contains state and not control sequences explicitly.

\subsection{Local Control Sequence Matching}
\label{sec:matching}

To minimize the loss  $\obj$, 
each iteration of \algo  applies RHC  with time horizon $K$ under the current cost parameters $\theta$ and computes  locally  optimal control sequences  of length $K$. In contrast, classic IRL algorithms, such as MaxEnt~\citep{MaxEntIRL}, perform global optimization over the entire task duration $T$ in the inner loop.
While RHC sacrifices global optimality, it is much more scalable and enables \algo to handle high-dimensional continuous systems. We use the hyperparameter $K$  to trade off optimality and scalability. 

Specifically, we use MPPI~\citep{MPPI} to solve for an optimal control sequence distribution  at the current start state $x_t$ in  iteration $t$. 
The main result of MPPI suggests that the optimal control sequence distribution $Q$ minimizes the ``free energy'' of the dynamical system and  this  free energy can be calculated from the cost of the state trajectory  under $Q$.  Mathematically, the probability density~$p_{\sss Q}(\optimalActControlSeq)$  can be 
expressed as a function of the state cost $\stateCostFunc(V, x_t;\theta) $,  with respect to a Gaussian ``base'' distribution $B(\baseControlInput, \noiseMatrix)$ that depends on the control cost:
\begin{equation}
\label{equ:optimal_dis}
    p_{\sss Q}(\optimalActControlSeq|\baseControlInput,  \noiseMatrix, x_t;\theta) = {1 \over \partition}\,{p_{\sss B}(\optimalActControlSeq|\baseControlInput, \noiseMatrix) \exp\bigl(-\frac{1}{\inverseTemp} S(\optimalActControlSeq, x_t;\theta)\bigr)}, 
\end{equation}
where 
$Z$ is the partition function. 
 For the quadratic control cost in \eqref{equ:cost}, we have $\baseControlInput = (0, 0, \ldots)$~\citep{MPPI}. 
Intuitively, the expression in \eqref{equ:optimal_dis} says that the  probability density of $\optimalActControlSeq$ is the product of two factors, one penalizing high control cost and one penalizing high state cost. So  controls with large values or  resulting in high-cost states occur with probability exponentially small. 

Equation \eqref{equ:optimal_dis} provides the optimal control sequence distribution under the current cost.  Suppose that the  control sequences for expert demonstrations $\expertDemo_t$ follow a distribution $E$.
We  define the loss $\obj(\theta; \expertDemo_t, x_t)$ as  the KL-divergence between the two distributions:
\begin{equation}
\obj(\theta; \expertDemo_t, x_t) = \dkl \bigl(p_{\sss E}(V|x_t) \parallel p_{\sss Q}(V|\baseControlInput,\noiseMatrix,x_t;\theta)\bigr)
    \label{kl},
\end{equation}
which \algo seeks to minimize in each iteration. 
While the loss $\obj$ is defined in terms of control sequence distributions, the expert demonstrations $\expertDemo$ provide state information and not control information. However, each control sequence $V$ induces a corresponding state sequence $\traj$ for a given start state $x_0$, and $\traj$ determines the cost of $V$ according to \eqref{equ:cost}. We show in the next subsection that 
 $\partial \obj / \partial \theta$ can be computed efficiently using only state information from $\expertDemo$. This makes \algo  appealing for learning tasks in which control labels are difficult or impossible to acquire. 
 
\subsection{Gradient Optimization}\label{sec:grad}
To simplify  notations, we remove the explicit dependency on  $x_t$ and $\expertDemo_t$ and assume that in iteration $t$ of \algo, all control sequences are applied with $x_t$ as the start state and expert demonstrations come from $\expertDemo_t$. 
We have 
\begin{align}
   {\partial \obj \over \partial \theta}
 & = {\partial \over \partial \theta}\int p_{\sss E}(V)\log\;\frac{p_{\sss E}(V)}{p_{\sss Q}(V|\baseControlInput, \noiseMatrix;\theta)}d V \nonumber \\ 
 &  =  \int p_{\sss E}(V)\Bigl(\frac{1}{\inverseTemp}\frac{\partial }{\partial \theta}S(V;\theta)\Bigr) dV - 
\int p_{\sss Q}(V|\baseControlInput, \noiseMatrix;\theta)\Bigl(\frac{1}{\inverseTemp}\frac{\partial }{\partial \theta}S(V;\theta)\Bigr)dV \
  \label{integral}
\end{align}
The first line  follows directly from the definition, and the derivation for the second line is available in Appendix~\ref{appendix:gradient}.

We  estimate the two integrals in  (\ref{integral}) through sampling. For the first integral, we can use the expert demonstrations as samples. For the second integral, we cannot sample $p_{\sss Q}$ directly, as the optimal control distribution $Q$ is unknown in advance. Instead, we sample from a known Gaussian distribution with density $p(V|U,\noiseMatrix)$ and apply importance sampling so that 
\begin{equation}
    \mathbb{E}_{p_{\sss Q}(V|\baseControlInput, \noiseMatrix;\theta)}[V] = \mathbb{E}_{p(V|U, \noiseMatrix)}[w(V)V].
     \label{equ:importance_sampling_control}
\end{equation}
The importance weight $w(V)$ is given by 
\begin{equation}
    w(V)  
    = \frac{p_{\sss Q}(V|\baseControlInput, \noiseMatrix;\theta)}{p(V|U,\noiseMatrix)}
    =  \frac{p_{\sss Q}(V|\baseControlInput, \noiseMatrix;\theta)}{p_{\sss B}(V|\baseControlInput, \noiseMatrix)} \times \frac{p_{\sss B}(V|\baseControlInput, \noiseMatrix)}{p(V|U, \noiseMatrix)} 
    \label{weight}
\end{equation}
To simplify the first ratio in (\ref{weight}), we  substitute in the expression for $p_{\sss Q}$ from  (\ref{equ:optimal_dis}):
\begin{equation}
    \frac{p_{\sss Q}(V|\baseControlInput, \noiseMatrix;\theta)}{p_{\sss B}(V|\baseControlInput, \noiseMatrix)}= \frac{1}{Z}\exp\Bigl(-\frac{1}{\lambda}(S(V;\theta)\Bigr)
    \label{ratio1}
\end{equation}
 We then simplify the second ratio, as both are Gaussian distributions with the same covariance matrix~$\noiseMatrix$:
\begin{equation}\label{ratio2}
    \frac{p_{\sss B}(V|\baseControlInput, \noiseMatrix)}{p(V|U, \noiseMatrix)}  \propto \exp\bigg(-\sum_{k=0}^{{K-1}}(u_k-u^{\sss B}_k)^{\intercal}\noiseMatrix^{-1}v_k)\bigg),
\end{equation}
where $u_k$ and $v_k$, $k=0,1,\ldots, K-1$ are the mean controls and the sampled controls from $p(V|U, \noiseMatrix)$, respectively. Similarly, $u_k^{\sss B}$, $ k=0,1,\ldots, K-1$ are the mean controls for the base distribution, and  they are all $0$ in our case.  
Substituting (\ref{ratio1}) and (\ref{ratio2}) into \eqref{weight},  we have 
\begin{equation}
    w(V) \propto \exp(-\frac{1}{\lambda}\bigg( S(V;\theta)+\lambda\sum_{k=0}^{K-1}u_k^{\intercal}\noiseMatrix^{-1}v_k\bigg)
\label{eqn:weights}
\end{equation}
For each sampled control sequence $V$, we evaluate the expression in (\ref{eqn:weights}) and normalize over all samples to obtain $w(V)$. 

To summarize, we estimate $\partial \obj / \partial \theta  $ through sampling:
\begin{align}\label{equ:grad}
    \frac{\partial}{\partial \theta}\obj(\theta; \expertDemo_t, x_t)& \approx \frac{1}{N}\sum_{i=1}^{N}\frac{1}{\inverseTemp}\frac{\partial}{\partial \theta}S(V_i, x_t;\theta)-\frac{1}{M}\sum_{j=1}^{M}\frac{1}{\inverseTemp}w(V_j)\frac{\partial }{\partial \theta}S(V_j,x_t;\theta),
\end{align}
where $V_i, i=1, \ldots, N$ are the control sequences for the expert demonstrations in $\expertDemo_t$ and $V_j, j=1,2, \ldots, M$ are the sampled control sequences. 
Equation~\eqref{equ:grad} connects $\partial \obj/\partial \theta$ with $\partial \stateCostFunc / \partial \theta$. The state cost function $\stateCostFunc$ is represented as a shallow neural network, and its derivative can be  obtained easily through back-propagation. To evaluate $\frac{\partial}{\partial \theta}S(V_i, x_t;\theta)$, we do not actually use the expert control sequences, as they are unknown. We use the corresponding state trajectories in $\expertDemo_t$ directly, as the state cost depends only on the visited states. See equation \eqref{equ:statecost}.

Finally, we  approximate the optimal mean control sequence according to   \eqref{equ:importance_sampling_control}:
\begin{equation}
    U = \mathbb{E}_{p_{\sss Q}(V|\baseControlInput, \noiseMatrix,x_t;\theta)}[V] \approx \sum_{j=1}^{M} w(V_j)V_j. 
\end{equation}
The first element in the 
control sequence $U$ is the chosen control for the current time step $t$. We remove the first element from $U$ and append zero at the end. The new control sequence is then used as the mean for the sampling distribution in the next iteration.

\subsection{Analysis}
\label{sec:analysis}
Since \algo performs local optimization sequentially over many steps, one main
concern is error accumulation over time. For example, standard behavior
cloning has one-step error $\epsilon$ and cumulative error $O(T^2 \epsilon)$
over $T$ steps, because of covariate shift~\citep{ross2011reduction}. The
DAgger algorithm reduces the error to $O(T \epsilon)$ by querying the expert
repeatedly during online learning~\citep{ross2011reduction}. We prove a similar
result for \algo, which uses offline expert demonstrations only.  In iteration
$t$ of \algo, let $p_{\sss E}(V_t|x_t)$ be the $K$-step expert demonstration
distribution and $p_{\sss Q}(V_t|\baseControlInput,\noiseMatrix,x_t;\theta)$
be the computed $K$-step optimal control distribution for some fixed cost
parameters $\theta$. \algo minimizes the KL-divergence between these two
distributions in each iteration. Let $\expertStateDistribution(x)$ be the
state marginal distribution of expert demonstrations and
$\RHCStateDistribution(x ; \theta)$ be the state marginal distribution of the
computed RHC policy under $\theta$ over the entire task duration
$T$. Intuitively, we want the control policy under the learned cost to visit
states similar to those of expert demonstrations in distribution. In other words,
$\RHCStateDistribution(x;\theta)$ and $\expertStateDistribution(x)$ are close.
\begin{theorem}\label{th:StateMarginalError}
If $D_{\sss \emph{KL}}\bigl(p_{\sss E}(V_t|x_t) \parallel p_{\sss Q}(V_t|\baseControlInput,\noiseMatrix,x_t;\theta)\bigr) < \epsilon$ for all $t = 0, 1, ..., T-1$, then 
$
    D_{\sss \emph{TV}}(\expertStateDistribution(x) , \RHCStateDistribution(x; \theta)) <  T\sqrt{\epsilon / 2}.
$
\end{theorem}
The theorem says that \algo's cumulative error, measured in total variation
distance $D_{\sss \textrm{TV}}$ between $\expertStateDistribution(x)$ and
$\RHCStateDistribution(x;\theta)$, grows linearly with $T$.  The proof
consists of three steps. First, in each iteration $t$, if the KL-divergence
between two control sequence distributions are bounded by $\epsilon$, so is
the KL-divergence between the two corresponding state distributions induced by
control sequences. Next, we show that the KL-divergence between the state
distributions over two successive time steps are bounded by the same
$\epsilon$. Finally, we switch from KL-divergence to total variation distance
and apply the triangle inequality to obtain the final result. Note that our
local optimization's objective is defined in KL divergence, while the final
error bound is in TV distance. We switch the distance measures to get the best
from both.  Minimizing the KL divergence leads to strong local optimization
result, but KL itself is not a proper metric. Therefore, we further bound the
KL divergence by TV distance to obtain a proper metric bound for the final
result. The full proof is given in Appendix~\ref{sec:proof}. Since RHC
performs local optimization in each iteration, we cannot guarantee global
optimality. However, the theorem indicates that unlike standard behavior
cloning, the cumulative error of \algo grows linearly and not quadratically in
$T$. This shows one advantage of IRL over behavior cloning from the
theoretical angle.

Given a control policy $V$ with the resulting state marginal distribution $p_{\sss V}(x)$, we can calculate the  expected total cost of $V$ by integrating the one-step cost over $p_{\sss V}$.  Now suppose that the one-step cost is bounded. Theorem~\ref{th:StateMarginalError} then
implies that the regret in total cost, compared with the expert policy, also grows linearly in $T$.

\subsection{Extension to Stochastic Dynamics}
Suppose that the system dynamics is stochastic:
$x_{t+1} = f(x_t, v_t, \omega_t)$, where $\omega_t$ is a random variable that
models the independent system noise. \algo still applies, with
modifications. We redefine the total cost functions
$\tilde{\controlSeqCostFunc}(V, \stateinit; \theta)$ and
$\tilde{\stateCostFunc}(V, \stateinit;\theta)$ by taking expectation over the
system noise distribution. When calculating the importance weight
$\tilde{w}(V)$ in \eqref{eqn:weights}, we sample over the noise distribution
to estimate the expected total state cost. Finally, we may need more samples
when estimating the gradient in~\eqref{equ:grad}, because of the increased
variance due to the system noise. The full derivation of the extension to
stochastic dynamics  is given in Appendix
\ref{Appendix:extension_stochastic}.  The experiments in
the current work  all have deterministic dynamics.
We leave experiments with the extension  to future work.

\section{Experiments}
\label{sec:expr}

\begin{figure*}
\centering
\hspace*{-15pt}
\small
\begin{tabular}{l@{\hspace{-5pt}}c@{\hspace{10pt}}c@{\hspace{15pt}}c@{\hspace{15pt}}c@{\hspace{15pt}}c@{\hspace{15pt}}c}
& \multicolumn{1}{c}{Pendulum-v0} &
 \multicolumn{1}{c}{LunarLander-v2} &
 \multicolumn{1}{c}{Hopper-v2} &
 \multicolumn{1}{c}{Walker2d-v2} &
 \multicolumn{1}{c}{Ant-v2} &
 \multicolumn{1}{c}{CarRacing-v0}
 \\
    &   
          \includegraphics[width=0.14\columnwidth]{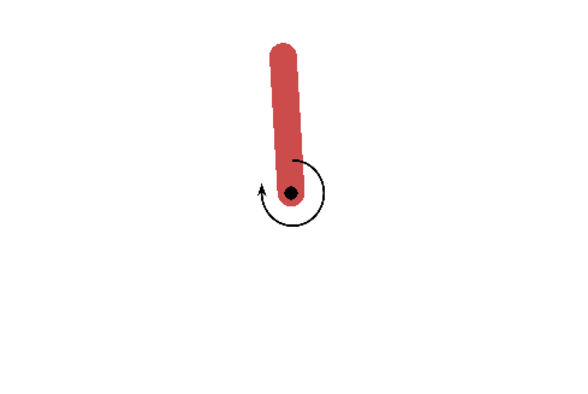}
    & 
          \includegraphics[width=0.14\columnwidth]{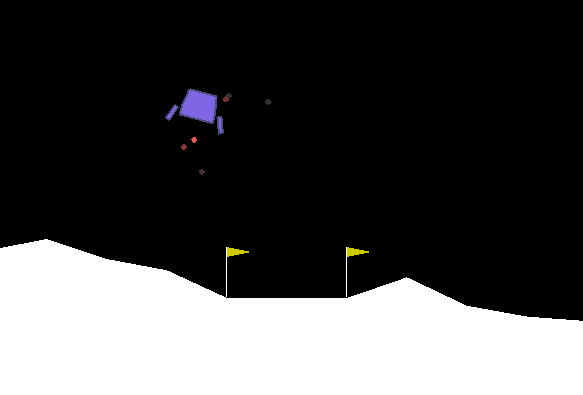}
    & 
          \includegraphics[width=0.14\columnwidth]{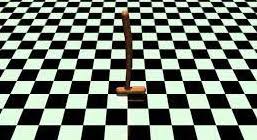}
    & 
          \includegraphics[width=0.14\columnwidth]{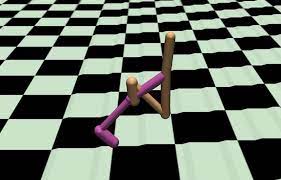}
    & 
          \includegraphics[width=0.14\columnwidth]{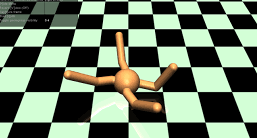}
    & 
          \includegraphics[width=0.14\columnwidth]{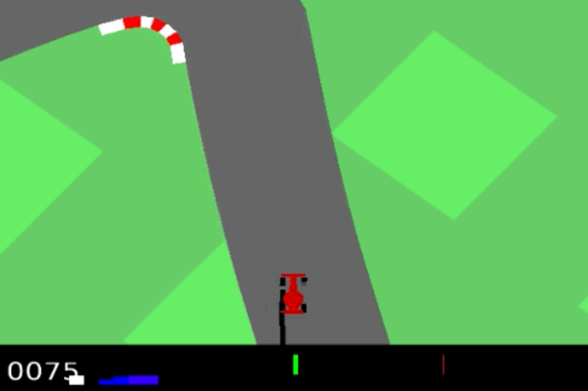}
    \\

$|\mathcal{S}|$: &3-D             &   8-D             &  11-D     &17-D             &    111-D            &     $96\times 96$ (image)        \\
$|\mathcal{A}|$:  & 1-D           &    2-D           &    ~~3-D & ~~6-D   &    ~~~~8-D             &     3-D      
 \\
 ~~$T$:  & 100         &    250         &    1000 & 1000   &   1000             &     1000    
 \\
\end{tabular}
\caption{Benchmark tasks. $|\mathcal{S}|$ and $|\mathcal{A}|$ denote the dimensions of the state space and the action space, respectively. $T$ denotes the task horizon. }\label{tasks}
\vspace*{-6pt}
\end{figure*}

We investigate two main questions. Does \algo scale up to high-dimensional continuous control tasks? Does \algo learn a robust cost function under noise? 

\subsection{Setup}
We compare \algo with two leading IRL algorithms, namely AIRL~\citep{AIRL} and $f$-IRL~\citep{fIRL}, and one imitation learning algorithm,  GAIL~\citep{GAIL}. In particular, $f$-IRL is a recent algorithm that achieves leading performance on high-dimensional control tasks.
We use the implementation of AIRL, GAIL, and $f$-IRL  from the $f$-IRL's official repository along with the reported hyperparameters~\citep{fIRL}, whenever possible. We also perform hyperparameter search on a grid to optimize the performance of every method on every task. 
The specific hyperparameter  settings used are reported in  Appendix \ref{Appendix:task}.  

Our benchmark set consists of six continuous control tasks (\figref{tasks}) from OpenAI Gym~\citep{brockman2016openai}, with increasing sizes of state and action spaces. For the most complex task, CarRacing, the input consists of   $96\times96$ raw images, resulting in an enormous state space that poses great challenges \citep{world}. To our knowledge, \algo is the first few IRL algorithms to attempt such a high-dimensional space. For fair comparison, we customize all tasks to the fixed task horizon settings (\figref{tasks}. See  Appendix~\ref{Appendix:task} for details on task parameter settings.

We use WorldModel~\citep{world} to generate expert demonstration data for CarRacing and use SAC~\citep{SAC} for the other tasks. 
We add Gaussian noise to the input controls and collect  expert demonstrations at different control noise levels. The covariance  of the control noise is \textit{unknown} to all methods, including, in particular, \algo. 

To measure the performance of the learned cost function and policy, we score its induced optimal policy using the ground-truth cost function. 
For ease of comparison with the literature, we use negated cost values, \ie, rewards, in all reported results. Higher values indicate better performance.
Each experiment is repeated 10 times to estimate the performance variance.

\subsection{Scalability}
\begin{wraptable}{r}{0.73\columnwidth}
\fontsize{7}{7}
\selectfont
\vspace*{-12pt}
\caption{Performance comparison of \algo and other methods. The performance is reported  as the ratio of the learned policy's average return and  the expert's average return. 
The absolute average returns and the standard deviations are reported in Appendix~\ref{sec:AddResults}.  Negative ratios are clipped to $0$.
The two numbers under the name of each environment indicate the dimensions of the state space and the action space, respectively.}
\centering
    \begin{tabular}{l@{\hspace*{6pt}}l@{\hspace*{10pt}}ccc}\toprule
        &&\textbf{\makecell{No Noise\\ $\Sigma = 0$}} & \textbf{\makecell{Mild Noise\\$\Sigma = 0.2$}} & \textbf{\makecell{High Noise\\$\Sigma = 0.5$}}\\\midrule
        Pendulum   & Expert & -154.69 $\pm$ 67.61 & -156.50 $\pm$ 70.72 & -168.54 $\pm$ 80.89\\
        ~~3, 1 & \algo & 1.06 & \textbf{1.07} & \textbf{1.08} \\
             & $f$-IRL & \textbf{1.07} & 1.06  & 0.93\\
             & AIRL & 1.05 & 0.94 & 0.91\\
             & GAIL & 0.88 & 0.89 & 0.80\\
        \midrule
       Lunarlander & Expert & 232.00 $\pm$ 86.12 & 222.65 $\pm$ 56.35 & 164.52 $\pm$ 16.79\\ 
         ~~8, 2   & \algo & \textbf{1.05} & \textbf{1.04} & \textbf{1.20}\\
             & $f$-IRL & 0.76 & 0.63 & 0.74\\
             & AIRL & 0.74 & 0.60 & 0.58\\
             & GAIL & 0.72& 0.56 & 0.60 \\
            
        \midrule
        Hopper & Expert & 3222.48 $\pm$ 390.65 & 3159.72 $\pm$ 520.00 & 2887.72 $\pm$ 483.93\\ 
         ~~11, 3    & \algo & 0.95 & \textbf{0.98} & \textbf{0.96}\\
             & $f$-IRL & \textbf{0.96}  & 0.82 & 0.43\\
             & AIRL & 0.01 & 0.01 & 0.01\\
             & GAIL & 0.82 & 0.50 & 0.24\\
             
         \midrule
         Walker2d & Expert & 4999.47 $\pm$ 55.99 & 4500.43 $\pm$ 114.48& 3624.48 $\pm$ 95.05\\
          ~~17, 6      & \algo & \textbf{0.99} & \textbf{0.99} & \textbf{0.95}\\
             & $f$-IRL & \textbf{0.99} & 0.82 & 0.78\\
             & AIRL & 0.00 & 0.00 & 0.00\\
             & GAIL & 0.50 & 0.64 & 0.51\\
        \midrule
        Ant  & Expert & 5759.22 $\pm$ 173.57 &  3257.37 $\pm$ 501.95 & 252.62 $\pm$ 91.44\\ 
          ~~111, 8      & \algo & 0.86 & \textbf{0.93} & \textbf{0.91}\\
             & $f$-IRL & \textbf{0.87} & 0.80 & 0.78\\
             & AIRL & 0.17 & 0.33 & 0.00\\
             & GAIL & 0.48 & 0.40 & 0.00\\
            
        \midrule
        CarRacing  & Expert & 903.25 $\pm$ 0.23 & 702.01 $\pm$ 0.3 & 281.12 $\pm$ 0.34\\
         ~~$96\times 96$, 3       & \algo & \textbf{0.40} & \textbf{0.29}  & \textbf{0.19}\\
             & $f$-IRL & 0.09 & 0.03 & 0.00\\
             & AIRL & 0.00 & 0.00 & 0.00\\
             & GAIL & 0.00 & 0.01 & 0.00\\
        \bottomrule
    \end{tabular}
\vspace{-0.5cm}
\label{tab:scalablity}
\end{wraptable}
We compare \algo with other methods, first in noise-free environments (\figref{fig:efficiency}) and then with increasing noise levels (\tabref{tab:scalablity}).

\begin{figure}
\centering
\captionsetup[subfigure]{font=scriptsize,labelfont=scriptsize,position=bottom,labelformat=empty}
\subfloat[Drive in the center of the lane. ]{\includegraphics[width=0.23\columnwidth]{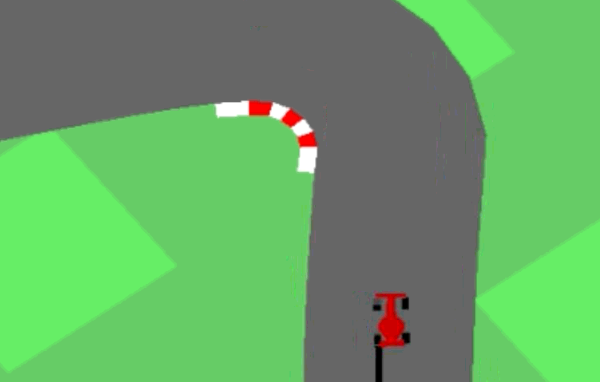}}
\hspace{0.02\columnwidth}
\subfloat[Adjust the direction at a sharp turn.]{\includegraphics[width=0.23\columnwidth]{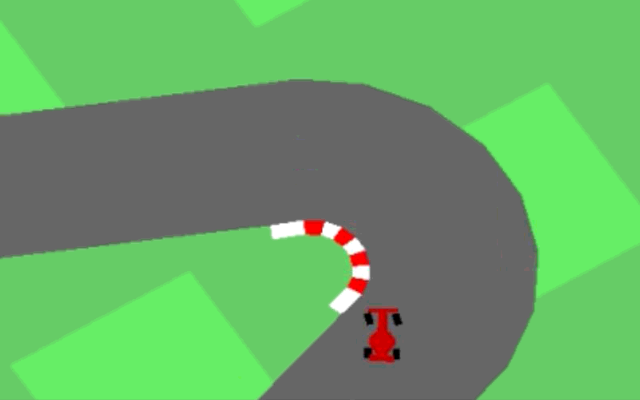}}
\hspace{0.02\columnwidth}
\subfloat[Choose the shortest path to make the turn.]{\includegraphics[width=0.23\columnwidth]{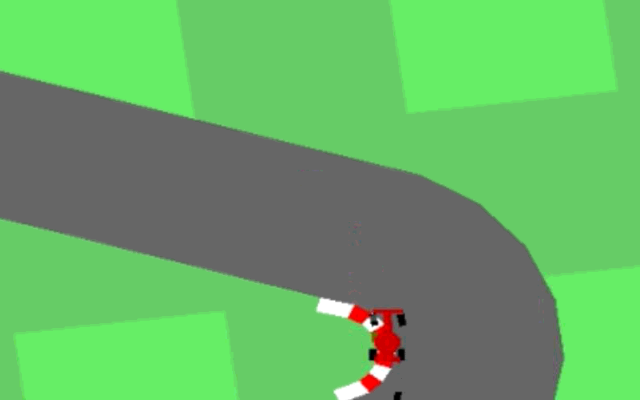}}
\hspace{0.02\columnwidth}
\subfloat[Align to the lane center after the turn.]{\includegraphics[width=0.23\columnwidth]{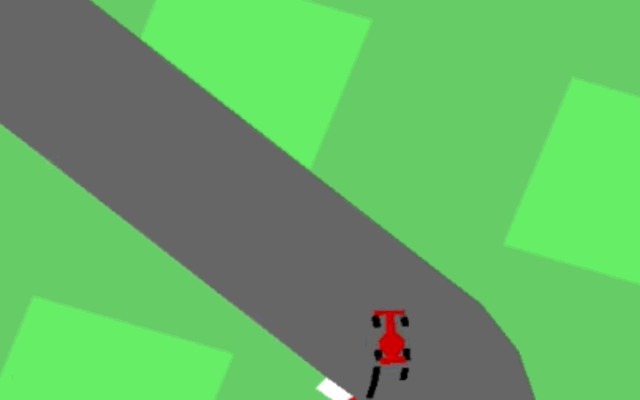}}
\caption{Driving behaviors learned by \algo for CarRacing-v0, with pixel-level raw image input.}
\label{fig:carracing}
\vspace{-10pt}
\end{figure}

\figref{fig:efficiency} shows the learning curve of each methods in noise-free environments. Overall, \algo converges faster and achieves higher return, especially for tasks with  higher state space dimensions. This improved performance suggests that the benefit of local optimization adopted by \algo outweighs its potential limitations. 

\tabref{tab:scalablity} shows the final performance of all methods at various noise levels. \algo clearly outperforms AIRL and GAIL in all experiments. So we focus our discussion on comparison with $f$-IRL. In noise-free environments, \algo and $f$-IRL perform comparably on most tasks. On CarRacing, the most challenging task, \algo performs much better. 
\algo manages to learn the critical driving behaviors  illustrated in \figref{fig:carracing}, despite the high-dimensional image input. However, \algo does not manage to learn to drive fast enough. That is the main reason why it under-performs the expert. 
In comparison, $f$-IRL only learns to follow a straight lane after  a large number of environment steps, and still fails to make a sharp turn after $3.0\times 10^7$ environment steps. In the  noisy environments, the advantage of \algo over $f$-IRL is more pronounced even on some of the low-dimensional tasks, because \algo accounts for the control noise explicitly in the cost function.

\begin{figure*}
\includegraphics[width=1\columnwidth]{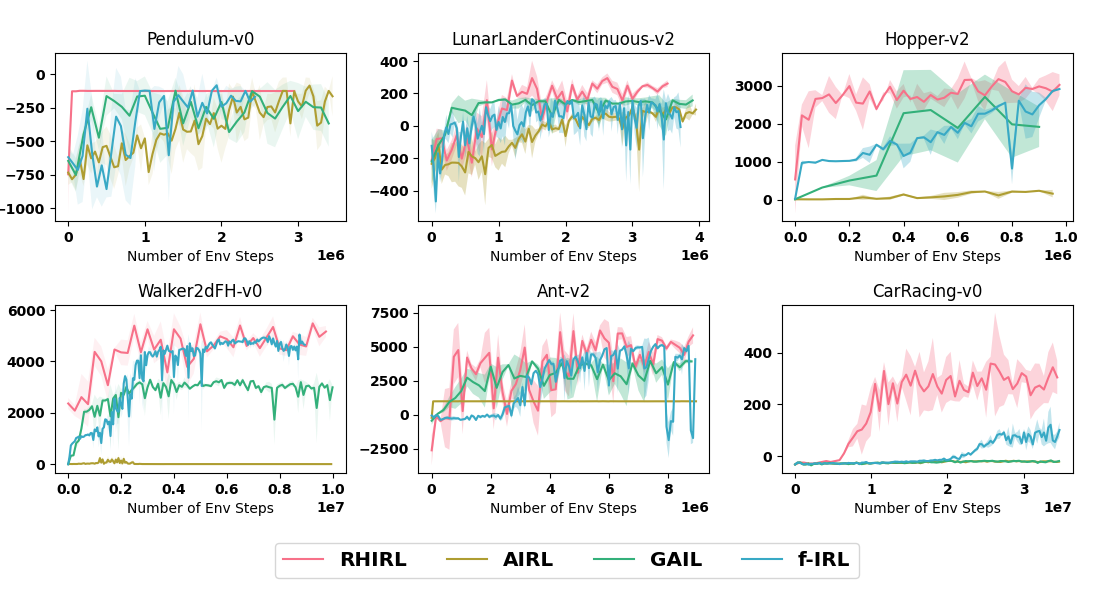}
\caption{Learning curves for \algo and other methods. 
}\label{fig:efficiency}
\vspace{-12px}
\end{figure*}

\subsection{Robustness}\label{sec:robustness}
\begin{wraptable}{r}{0.6\columnwidth}
\vspace{-12pt}
\fontsize{7.5}{7.5}
\selectfont
\caption{Robustness of \algo and other methods under noise. 
The performance is measured as the ratio between the average return of an re-optimized policy in a noisy environment and the expert's average return in the corresponding noise-free environment. The absolute average returns and the standard deviations are reported in Appendix~\ref{sec:AddResults}. The negative ratios are clipped to $0$.}
\centering
\begin{tabular}{lc c c c}\toprule
 & \multirow{2}{*}{}{} & \multicolumn{3}{c}{\textbf{\makecell{~~~~~~~~~~Noise Level $\Sigma$  }}} \\
 & &\textrm{0.0}& \textrm{0.2}   & \textrm{0.5}    \\
                          \midrule
Pendulum & Expert  &-154.69 $\pm$ 67.61 & -- & -- \\
   ~~3, 1         & \algo & 1.06  & \textbf{1.07} & \textbf{1.06} \\ 
            & $f$-IRL & \textbf{1.08} & 0.90 & 0.85  \\
            & AIRL & 1.05 & 0.79 & 0.67 \\
            & GAIL  & 0.88 & 0.71 & 0.62\\
                          \midrule
LunarLander & Expert & 232.00 $\pm$ 86.12 & -- & --\\ 
   ~~8, 2         & \algo & \textbf{1.05} & \textbf{0.89}  & \textbf{0.76}\\
            & $f$-IRL & 0.76 & 0.53  & 0.44 \\
            & AIRL  & 0.74 & 0.14 & 0.10  \\
            & GAIL  & 0.72  & 0.44 & 0.34\\
                          \midrule
Hopper & Expert & 3222.48 $\pm$ 390.65 & -- & -- \\ 
  ~~11, 3          & \algo & 0.95 & \textbf{0.80}  & \textbf{0.67}\\
            & $f$-IRL  &\textbf{0.96}&  0.65  & 0.62\\
            & AIRL  &0.01&  0.01  & 0.00\\
            & GAIL &0.82 & 0.07 & 0.06\\
                          \midrule
Walker & Expert  & 4999.47 $\pm$ 55.99 & -- & -- \\
  ~~17, 6          & \algo &\textbf{0.99}& \textbf{0.80}  & \textbf{0.69}\\
            & $f$-IRL  &\textbf{0.99}& 0.60 & 0.22\\
            & AIRL  &0.00& 0.28 & 0.36\\
            & GAIL &0.50 &  0.02 & 0.02\\
                          \midrule
Ant & Expert  & 5759.22 $\pm$ 173.57 & -- & -- \\ 
  ~~111, 8          & \algo &\textbf{0.86} & \textbf{0.55} & \textbf{0.15} \\
            & $f$-IRL &0.87 & 0.35  & 0.08\\
            & AIRL & 0.17 & 0.15 & 0.00 \\
            & GAIL  & 0.48 & 0.00 & 0.00\\
             \midrule
CarRacing  & Expert & 903.25 $\pm$ 0.23 & -- & -- \\
  ~~$96\times 96$, 3         & \algo &\textbf{0.40}& \textbf{0.29} & \textbf{0.12}\\
            & $f$-IRL  &0.09& 0.02 & 0.00\\
            & AIRL  &0.00& -- & -- \\
            & GAIL &0.00  & -- & --\\
\bottomrule
\end{tabular}
\vspace{-0.8cm}
\label{tab:noise}
\end{wraptable}

Next we evaluate \algo and other methods for robustness under noise. 
A robust cost function encodes the expert's true intent. It is  expected to yield consistent performance over different noise levels, regardless of noise in expert demonstrations.

For each task, a cost function is learned in a noise-free environment and is then used to re-optimize a policy in the corresponding noisy environments.
Specifically for GAIL, since it learns a policy and does not recover the associated cost function, we directly apply the learned  policy in  noisy environments. 

\tabref{tab:noise} shows that noise causes performance degradation in all methods. However, \algo is  more robust in comparison. For simple  tasks,  Pendulum and Lunarlander, \algo and $f$-IRL perform consistently well across different noise levels, while GAIL and AIRL fail to maintain their good performance, when the noise level increases. 
For the more challenging tasks, Hopper and Walker,  \algo's performance degrades mildly, and $f$-IRL suffers more significant performance degradation. 
It is worth noting that the expert demonstrations used to train the transferred cost function are from the perfect system. Therefore, some expert actions and states may no longer be optimal or feasible in a highly noisy environment. Moreover, the cost function trained in the perfect system cannot reason about the long-term consequences of an action in a high noise environment. Therefore, it is challenging for the learned cost function to be robust to a highly noisy environment, as capturing the true intention of the expert is difficult. 

\subsection{Effect of Receding Horizon $K$}
\label{sec:ablate_hopper}
\begin{wrapfigure}{r}{0.4\columnwidth}
\vspace{-25pt}
\centering
\captionsetup[subfigure]{font=scriptsize,labelfont=scriptsize,labelformat=empty}
\includegraphics[width=0.4\columnwidth]{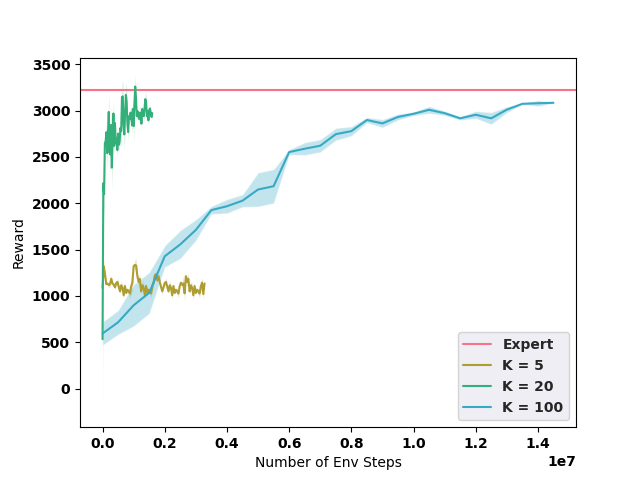}
\caption{The effect of receding horizon $ K$ on the performance of \algo on the Hopper-v2 task.}\label{fig:ablation_hopper}
\end{wrapfigure}
\algo uses the receding horizon $K$ to trade off optimality and efficiency. We hope to ablate the effect of K on Hopper-v2 to show how different $K$s affect the final performance and sample complexity. The task horizon for Hopper-v2 is 1000 steps, \ie $T= 1000$. We run \algo with the receding horizon $K \in \{5, 20, 100\}$. The results are illustrated in \figref{fig:ablation_hopper}. When $K$ is small, \algo improves its performance quickly but converges to the suboptimal solution. For $K = 5$, \algo's performance shoots up after the first few iterations to $1000$, then it quickly converges to a final score of $1100$. When $K$ increases, though the performance improves slightly slower than  $K=5$, it can continue to learn and reach a score of $3071.68$. At $K = 20$, it takes fewer than $1e6$ env steps to stabilize to a score greater than $3000$. However, when $K$ is too large, the learning becomes much slower. When $K= 100$, it takes more than $1e7$ env steps to stabilize to a score larger than $3000$, which is $10$ times more than when $K=20$. On the other hand, $K=100$ can achieve a final score of $3083$, which is slightly more than that of $K=20$. This ablation study shows that our receding horizon $K$ can tradeoff optimality and efficiency: using a smaller $K$ allows us to learn faster at the expense of a sub-optimal solution, while using a large $K$ may make the learning inefficient. Seeking a suitable $K$ can balance the requirement for optimality and efficiency.

\section{Conclusion}\label{sec:conc}
\algo is a scalable and robust IRL algorithm  for high-dimensional, noisy, continuous systems. Our experiments show that \algo outperforms several leading IRL algorithms  on multiple benchmark tasks, especially when expert demonstrations are noisy.
 
\algo's choice of local rather than global optimization is an important issue that deserves  further investigation. 
Overall, we view this as an interesting trade-off between scalability and optimality.  While this trade-off  is well known in reinforcement learning, optimal control, and  general optimization problems, it is mostly unexplored in IRL. Further, local optimization may tie the learned cost with the optimizer. It would be interesting to examine whether the learned cost transfers easily to other domains with different optimizers. 
We are keen to investigate these important issues and their implications to IRL as our next step. 

\paragraph{Acknowledgments}
This research is supported in part by the Agency of Science, Technology and
Research, Singapore, through the National Robotics Program (Grant No. 192 25
00054), the National Research Foundation (NRF), Singapore and DSO National
Laboratories under the AI Singapore Program (AISG Award No:
AISG2-RP-2020-016), and the National University of Singapore (AcRF grant
A-0008071-00-00).  Any opinions, findings and conclusions or recommendations
expressed in this material are those of the authors and do not reflect the
views of NRF Singapore.

\bibliographystyle{abbrv}
\bibliography{example_paper}


\newpage
\appendix
\setcounter{page}{1}

\section{Gradient Derivation} \label{appendix:gradient}
This section derives the gradient ${\partial \obj \over \partial \theta}$ in equation (\ref{equ:grad}) in full details. Recall $\obj$ is defined as:
\begin{equation*}
    \obj = \int p_{\sss E}(V)\log\;\frac{p_{\sss E}(V)}{p_{\sss Q}(V|\baseControlInput, \noiseMatrix;\theta)}d V
\end{equation*}
Firstly, we substitute $\obj$ in ${\partial \obj \over \partial \theta}$ and rewrite $\log\;\frac{p_{\sss E}(V)}{p_{\sss Q}(V|\baseControlInput, \noiseMatrix;\theta)}$ as the difference between $\log\;p_{\sss E}(V)$ and $ p_{\sss Q}(V|\baseControlInput, \noiseMatrix;\theta)$.
\begin{align}\label{equ:grad1}
    \frac{\partial \obj}{\partial \theta}
    = &  \frac{\partial }{\partial \theta} \int p_{\sss E}(V)\log\;\frac{p_{\sss E}(V)}{p_{\sss Q}(V|\baseControlInput, \noiseMatrix;\theta)}dV \nonumber\\
    = &  \int p_{\sss E}(V) \frac{\partial }{\partial \theta}\log\; p_{\sss E}(V) dV -  \int p_{\sss E}(V) \frac{\partial }{\partial \theta} \log\;p_{\sss Q}(V|\baseControlInput, \noiseMatrix;\theta) dV \nonumber\\
    = &   -  \int p_{\sss E}(V) \frac{\partial }{\partial \theta} \log\; p_{\sss Q}(V|\baseControlInput, \noiseMatrix;\theta)dV 
\end{align}

Since $p_{\sss E}(V)$ is independent of $\theta$, the first derivative in the second line of equation (\ref{equ:grad1}) evaluates to 0. Next, we substitute $p_{\sss Q}(V|\baseControlInput, \noiseMatrix;\theta)$ using the optimal control sequence distribution expression in equation (\ref{equ:optimal_dis}):
\begin{align}\label{equ:grad2}
   \frac{\partial \obj}{\partial \theta}
     = & -  \int p_{\sss E}(V)\frac{\partial}{\partial \theta} \log\;  \frac{1}{\partition}p_{\sss B}(V|\baseControlInput, \noiseMatrix)\exp(-\frac{1}{\inverseTemp} S(V;\theta)) dV \nonumber\\
    = & -\int p_{\sss E}(V)\frac{\partial}{\partial \theta}\log p_{\sss B}(V|\baseControlInput, \noiseMatrix)\;dV -\int p_{\sss E}(V)\frac{\partial}{\partial \theta}\log \exp(-\frac{1}{\inverseTemp} S(V;\theta))dV \nonumber \\ &+ \int p_{\sss E}(V)\frac{\partial }{\partial \theta}\log \partition dV  \nonumber\\
     = &  -\int p_{\sss E}(V)\frac{\partial }{\partial \theta} (-\frac{1}{\inverseTemp} S(V;\theta))dV + \int p_{\sss E}(V) \frac{\partial}{\partial \theta}\log \partition dV  
\end{align}
Since $p_{\sss B}(V|\baseControlInput, \noiseMatrix)$ is independent of $\theta$, the first derivative in the second line of equation (\ref{equ:grad1}) evaluates to 0. We are left with only two integrals in equation (\ref{equ:grad2}). 

Next, we factorize out $\frac{\partial}{\partial \theta} \log \partition$ from the integral since the partition function $\partition$ is constant to all $V$: 
\begin{align} \label{equ:grad3}
\frac{\partial \obj}{\partial \theta}
    = &  -\int p_{\sss E}(V)(-\frac{1}{\inverseTemp}\frac{\partial }{\partial \theta}S(V;\theta))dV +  (\frac{\partial}{\partial \theta}\log \partition)\int p_{\sss E}(V)dV  \nonumber\\
    = &  -\int p_{\sss E}(V)(-\frac{1}{\inverseTemp}\frac{\partial}{\partial \theta}S(V;\theta))dV +  \frac{1}{\partition}\frac{\partial \partition }{\partial \theta}
\end{align}
The second line in equation (\ref{equ:grad3}) follows as $\int p_{\sss E}(V)dV = 1$.

Next, we substitute $\partition = \int p_{\sss B}(V|\baseControlInput, \noiseMatrix)\exp(-\frac{1}{\inverseTemp} S(V;\theta)) dV$ in $\frac{\partial \partition  }{\partial \theta}$ and simplify it:
\begin{align}\label{equ:grad4}
     \frac{\partial \obj}{\partial \theta} = &  -\int p_{\sss E}(V)(-\frac{1}{\inverseTemp}\frac{\partial}{\partial \theta}S(V;\theta))dV +  \frac{1}{\partition}\frac{\partial}{\partial \theta} \int p_{\sss B}(V|\baseControlInput, \noiseMatrix)\exp(-\frac{1}{\inverseTemp} S(V;\theta)) dV  \nonumber\\
    = &  -\int p_{\sss E}(V)(-\frac{1}{\inverseTemp}\frac{\partial }{\partial \theta}S(V;\theta))dV +  \frac{1}{\partition}\int p_{\sss B}(V|\baseControlInput, \noiseMatrix)\frac{\partial}{\partial \theta}\exp(-\frac{1}{\inverseTemp} S(V;\theta)) dV  \nonumber\\
    = &  -\int p_{\sss E}(V)(-\frac{1}{\inverseTemp}\frac{\partial}{\partial \theta}S(V;\theta))dV\nonumber \\ 
    & + \frac{1}{\partition}\int p_{\sss B}(V|\baseControlInput, \noiseMatrix)\exp(-\frac{1}{\inverseTemp} S(V;\theta))\frac{\partial}{\partial \theta}(-\frac{1}{\inverseTemp} S(V;\theta))dV \nonumber \\
     = &  -\int p_{\sss E}(V)(-\frac{1}{\inverseTemp}\frac{\partial}{\partial \theta}S(V;\theta))dV\nonumber \\
     & + \int \frac{p_{\sss B}(V|\baseControlInput, \noiseMatrix)\exp(-\frac{1}{\inverseTemp} S(V;\theta))}{\partition}\frac{\partial}{\partial \theta}(-\frac{1}{\inverseTemp} S(V;\theta)) dV  
\end{align}

We rewrite $\frac{1}{\partition}p_{\sss B}(V|\baseControlInput, \noiseMatrix)\exp(-\frac{1}{\inverseTemp} S(V;\theta))$ in the last line of equation (\ref{equ:grad4}) as $p_{\sss Q}(V|\baseControlInput, \noiseMatrix;\theta)$ using equation (\ref{equ:optimal_dis}) and finally we have: 
\begin{align}
    \frac{\partial \obj}{\partial \theta}
    = & \int p_{\sss E}(V)(\frac{1}{\inverseTemp}\frac{\partial}{\partial \theta}S(V;\theta)) dV - \int p_{\sss Q}(V|\baseControlInput, \noiseMatrix;\theta)(\frac{1}{\inverseTemp}\frac{\partial}{\partial \theta}S(V;\theta))dV
\end{align}

\section{Theorem~\ref{th:StateMarginalError} and Proof}
\label{sec:proof}
\label{appendix1}
We use $T$ to denote the task horizon, and $K$ to denote the receding time horizon for local optimization. 
Let us simplify notations in the optimal control sequence distribution $p_{\sss Q}(V|\baseControlInput,\noiseMatrix,x_t;\theta)$ and remove the explicit dependency on  $\baseControlInput$ and $\noiseMatrix$. We assume that all control sequences are applied with $\baseControlInput$ as the base distribution and $\noiseMatrix$ as the covariate matrix. We assume the expert's underlying cost function is parameterized by $\thetaE$, so we have $p_{\sss E}(.) = p(.; \thetaE)$.

\subsection{Sketch}
We present a theoretical analysis on the convergence of \algo. Our main theorem \ref{th:StateMarginalError} states that given that the Kullback–Leibler (KL) divergence over local control sequence distribution for each time step $t=0, 1, ... T-1$ is bounded by $\epsilon$, though we do not query expert during the learning, using the cached expert demonstrations alone allows us to bound the error over global state marginal distribution \textit{linear} in the task horizon $T$ under total variance measure. 

First, we show that if the KL-divergence over the local control sequence distribution is bounded by $\epsilon$, so is the KL-divergence over the resulting state distribution.  At each time step $t$, the optimal control sequences distribution $p_{\sss Q}(V_t|x_t;\theta_t)$ at the initial state $x_t$ contains the full information to generate the corresponding state trajectories $p(\tau_t|x_t;\theta)$, and consequently the state distributions $p_t(x|x_t;\theta)$ (by neglecting the temporal information). Upon applying the information loss (lemma \ref{lemma1}), we prove in lemma \ref{lemma2} that, given initial state $x_t$, if the KL-divergence over $V_t$ is bound by $\epsilon$, so is the corresponding state distribution $p_t(x)$, \ie $ D_{\sss \textrm{KL}}(p_t(x| x_t; \thetaE)\parallel p_t(x|x_t;\theta)) \leq D_{\sss \textrm{KL}}(p(V_{t}|x_t; \thetaE)\parallel p_{\sss Q}(V_{t}|x_t;\theta)) \leq \epsilon$.

Next, we show the KL-divergence over global state marginal distribution between two consecutive time steps is also bounded by $\epsilon$. For each optimal control sequence we compute, we only execute the first control and use the rest to warm start the re-planning for the next time step. Therefore, for each time step, we only change a small region of the global state distribution, \ie reachable space of the current time step. We use $\RHCStateDistribution^t(x; \theta)$ to denote the global state marginal distribution by recursively applying RHC from time step $i = 0, ..., t$ under $\theta$ and switching to $\thetaE$ thereafter until $T-1$. Using generalized log sum inequality, we prove in lemma \ref{lemma3} that if the KL-divergence over $V_t$ is bounded by $\epsilon$ for all $t=0,...,T-1$, the KL-divergence over the global state marginal distribution between each of the two consecutive time steps is bounded by $\epsilon$, \ie $D_{\sss \textrm{KL}}(\RHCStateDistribution^t(x;\theta)\parallel \RHCStateDistribution^{t+1}(x;\theta)) \leq \frac{K+1}{T}\epsilon$.

Finally, we use Pinsker's inequality to upper bound the total variation (TV) distance by KL-divergence over state marginal distribution. Then we use the triangle inequality to show that the TV distance between expert and the actual visited state distribution over the task horizon $T$ using \algo is bounded by an error linearly in $T$, \ie  $D_{\sss \textrm{TV}}(p_{\sss E}(x)\parallel \RHCStateDistribution(x;\theta)) <  T\sqrt{\epsilon/2} $.

\subsection{Proofs}

First, we use \textbf{Lemma \ref{lemma1}, \ref{lemma2}} to prove that the KL-divergence over control sequence space upper bounds the KL-divergence over the resulting state distribution. We define the control sequence starts at task time step $t$ as $V_{t} = \{v_t, v_{t+1}, ..., v_{t+K-1}\}$. Moreover, its corresponding trajectory segment  $\tau_t = \{x_{t}, x_{t+1}, ..., x_{t+K}\}$ is computed uniquely from $V_t$ and initial state $x_t$ by iteratively applying the dynamic model $x_{t+1} = f(x_t, v_t)$. We use $p(x_t)$ to denote the state density at a single time step $t$. Assume $V_t$ is optimized based on the cost parameterization $\theta$, then the corresponding state distribution is defined as the summation of all state density over horizon $K$, \ie $p_t(x; \theta)=\frac{1}{K+1}\sum_{i=t}^{t+K} p(x_i;\theta)$. 

\begin{lemma}[Information loss\citep{f_D}]
Let $a$ and $b$ be two random variables and $f(.)$ be a convex
function. Let $P(a, b)$ be a joint probability distribution. The marginal distributions are $P(a) =
\sum_b P(a, b)$ and $P(b) = \sum_a P(a, b)$. Assume that $a$ can explain away $b$. This is expressed as follows – given any two probability distribution $P(.)$, $Q(.)$, assume the following equality holds for all $a$, $b$: 
\begin{equation}
    P(b|a) = Q(b|a)
\end{equation}
Under these conditions, the following inequality holds:
\begin{equation}
    \sum_a Q(a)f(\frac{P(a)}{Q(a)}) >  \sum_b Q(b)f(\frac{P(b)}{Q(b)}) 
\end{equation}
\label{lemma1}
\end{lemma}

\begin{lemma} \label{lemma2}
    Given the initial state $x_t$ and the two control sequence distributions $p(V_{t}|x_t; \thetaE)$ and $p_{\sss Q}(V_{t}|x_t;\theta)$, the KL-divergence between the resulting state distribution is upper bounded by KL-divergence between the control sequence distribution.
\begin{equation}
         D_{\sss \textrm{KL}}(p_t(x|x_t; \thetaE)\parallel p_t(x|x_t;\theta)) \leq D_{\sss \textrm{KL}}(p(V_{t}|x_t; \thetaE)\parallel p_{\sss Q}(V_{t}|x_t;\theta))
    \end{equation}
\end{lemma}

\begin{proof}
 Firstly, we prove that the KL-divergence over state trajectory distribution $p(\tau_t)$ is upper bounded by the KL-divergence between $p(V_t)$.
Given dynamical model $f$, the control sequence $V_t$ and the initial state $x_t$ contains all information to generate the corresponding $\tau_t$. Therefore, for any joint distribution $P(\tau_t, V_t|x_t)$ and $Q(\tau_{t}, V_{t}|x_t)$, the following is true
\begin{equation*}
    P(\tau_{t}| V_{t}, x_t) = Q(\tau_{t}| V_{t}, x_t)
\end{equation*}
Upon applying the information loss Lemma \ref{lemma1}, we have the inequality:
\begin{equation} \label{equ:traj_bound}
     D_{\sss \textrm{KL}}(p(\tau_{t}|x_t; \thetaE)\parallel p(\tau_{t}|x_t;\theta))
     \leq D_{\sss \textrm{KL}}(p(V_{t}|x_t; \thetaE)\parallel p_{\sss Q}(V_{t}|x_t;\theta))
\end{equation}

Next we prove that the KL-divergence between state distribution is upper bounded by the trajectory distribution. Since a trajectory $\tau_t = \{x_t, x_{t+1}, ..., x_{t+K}\}$ contains full information of the resulting states (by neglecting the temporal information), for any joint distribution $P(x|\tau_t)$ and $Q(x|\tau_t)$, the following is true
\begin{equation}
    P(x| \tau_t) = Q(x|\tau_t)
\end{equation}
Upon applying Lemma \ref{lemma1} we have the inequality:
\begin{equation}\label{equ:state_bound}
    D_{\sss \textrm{KL}}(p_{t}(x|x_t; \thetaE) \parallel p_{t}(x|x_t;\theta)) \leq D_{\sss \textrm{KL}}(p(\tau_t|x_t; \thetaE)\parallel p(\tau_{t}|x_t;\theta)) 
\end{equation}

Therefore, given the KL-divergence between the control sequence distribution is upper bounded by $\epsilon$, we use the equality in equation (\ref{equ:traj_bound}) and (\ref{equ:state_bound}) to show that the KL-divergence between the resulting state distribution is also upper bounded by $\epsilon$.

\begin{equation}
         D_{\sss \textrm{KL}}(p_t(x|x_t; \thetaE)\parallel p_t(x|x_t;\theta)) \leq D_{\sss \textrm{KL}}(p(V_{t}|x_t; \thetaE)\parallel p_{\sss Q}(V_{t}|x_t;\theta))\leq \epsilon
\end{equation}
\end{proof}

\begin{definition}[One-step Recoverability \citep{feedbackIM}]
Assume that the state distribution of the learner and expert are different at time $t$, that is $D_{\sss \textrm{KL}}(p(x_t;\theta_E)\parallel p(x_t;\theta)) \neq 0$, there exists a policy $\pi_{re}$ that when used for the learner, can bound:
\begin{equation}
    D_{\sss \textrm{KL}}(p(x_{t+1};\theta_E)\parallel p(x_{t+1};\pi_{re})) \leq \epsilon_1
\end{equation}
where the current initial state distribution of the student follows $p(x_t;\theta)$.
\end{definition}
Intuitively, this condition requires that, no matter what is the current state distribution, the learner can recover to the expert demonstrated distribution in a single time-step. In our case, this is a natural condition since the difference in the initial state distribution $p(x_t;\theta)$ and $p(x_t; \theta_E)$ is not arbitrarily large: we use re-planning to ensure the receding state sequences is always bounded below $\epsilon$, hence this recoverability condition can be easily satisfied. We emphasize that this recoverable policy is \textbf{never} executed in our algorithm, it is only used for the theoretical analysis. 

Next, we derive a bound over the global state marginal distribution between two consecutive time steps. At each time step $t=0,..., T-1$, we re-optimize the local control sequence distribution and only execute the first control, hence we only change state density over a small reachable space. We define $\RHCStateDistribution^t(x, \theta)$ as the global marginal state distribution by applying RHC from $i=0,1, ..., t$ under $\theta$ and then using the recoverable policy $\pi_{re}$ to switch to the expert $\thetaE$ thereafter until $T-1$, \ie, $\RHCStateDistribution^t(x; \theta) = \frac{1}{T} (\sum_{i=0}^{t+1} p(x_i; \theta) + p(x_{t+2}; \pi_{re}) + \sum_{i=t+3}^{T} p(x_i; \thetaE))$. According to the definition of the recoverable policy, we have $D_{\sss \textrm{KL}}(p(x_{t+2};\theta_E)\parallel p(x_{t+2};\pi_{re})) \leq \epsilon_1$, therefore, $\RHCStateDistribution^t(x; \theta) \approx \frac{1}{T} (\sum_{i=0}^{t+1} p(x_i; \theta) + \sum_{i=t+2}^{T} p(x_i; \thetaE))$. To quantify the change in global state marginal distribution, we derive a bound for the KL-divergence between two consecutive time steps, \ie $D_{\sss \textrm{KL}}(\RHCStateDistribution^{t-1}(x;\theta)\parallel \RHCStateDistribution^{t}(x;\theta))$.

\begin{lemma}\label{lemma3}
     If the KL-divergence over resulting state density from the control sequence distribution of length $K$ are bounded by $\epsilon$, \ie $D_{\sss \emph{KL}}(p_t(x|x_t; \thetaE) \parallel p_t(x|x_t;\theta)) < \epsilon$, where $x_t$ is the state encountered by our policy at $t = 0, 1, ..., T-1$ and is one-step recoverable, then KL-divergence over the global state marginal distribution between two consecutive control executions are bounded by $\epsilon$,
    \begin{equation}
         D_{\sss \textrm{KL}}(\RHCStateDistribution^{t-1}(x;\theta)\parallel \RHCStateDistribution^{t}(x;\theta)) \leq \frac{K+1}{T}\epsilon
    \end{equation}\label{corollary4}
    for $t = 1, ..., T-1$.
\end{lemma}
We state the generalized log sum inequality below in lemma \ref{lemma4}, the proof can be found in the Appendix of \citep{f_D}. Lemma \ref{lemma4} and \ref{lemma5} will be used in the proof for Lemma \ref{lemma3}.
\begin{lemma}[Generalized log sum inequality\citep{f_D}]\label{lemma4}
     Let $p_1, ..., p_n$ and $q_1, ..., q_n$ be non-negative numbers. Let $p = \sum_{i=1}^n p_i$ and $q = \sum_{i=1}^n q_i$. Let $f(.)$ be a convex function. We have the following:
     \begin{equation}
       \sum_{i=1}^n q_i f(\frac{p_i}{q_i}) \ge qf(\frac{p}{q})
     \end{equation}
\end{lemma}

\begin{lemma}\label{lemma5}
Let $p(x)$ and $q(x)$ be non-negative functions, and $c$ is a constant factor. We have the following:
\begin{equation}
     \int cp(x) \log \frac{cp(x)}{cq(x)} = c \int p(x) \log \frac{p(x)}{q(x)}
\end{equation}
\end{lemma}
\begin{proof}
\begin{equation}
    \int cp(x) \log \frac{cp(x)}{cq(x)}dx = \int cp(x) \log \frac{p(x)}{q(x)} dx = c\int p(x) \log \frac{p(x)}{q(x)} dx 
\end{equation}
\end{proof}

Now, we are ready to prove lemma \ref{lemma3}.
\begin{proof}
For each $t=0, ..., T-1$, we re-plan for the optimal local control sequence start at $x_t$ so that the resulting state distribution over horizon $K$ is bounded, \ie $D_{\sss \textrm{KL}}(p_{t}(x|x_t;\thetaE)\parallel p_{t}(x|x_t;\theta))<\epsilon$, where $p_t(x|x_t; \theta) = \frac{1}{K+1}\sum_{i=t}^{t+K}p(x_i;\theta)$. However, instead of executing all $K$ controls in the sequence, we only execute the first control at the current time step $t$ and change the state distribution $p(x_{t+1}; \theta)$ reachable for that single time step, then we use the remaining control sequence to warm start the local control sequence optimization for the next time step. To account the effect of replanning, for each time step $t$, since we do not change the state distribution after $p(x_{t+1};\theta)$, we can think of the change in the global state distribution as if we follow the optimal control under $\theta$ at time step $t$ and then use the recoverable policy $\pi_{re}$ to switch to $\thetaE$ afterwards over the control sequence horizon $K$. Hence, the actual state density is $\frac{1}{K+1}(p(x_t; \theta)+p(x_{t+1};\theta)+ p(x_{t+2};\pi_{re})+\sum_{i=t+3}^{t+K}p(x_i;\thetaE))$. Theoretically, since we do not query the expert online, the initial state $x_t$ distribution in theorem \ref{th:StateMarginalError} and lemma \ref{lemma2} should follow the expert demonstration at time $t$, \ie $p(x_t; \theta_E)$. However, our MPC controller cannot jump to this distribution and we replan from our current state distribution $p(x_t; \theta)$. To resolve this mismatch, we require the recoverability condition in our optimization procedure such that, for each resulting state from the controller, there always exists a one-step recoverable policy $\pi_{re}$ that can correct the current state distribution $p(x_t; \theta)$ to $p(x_{t+1}; \theta_E)$ in one step. Therefore, with this one-step recoverable condition on every state $x_t$ induced by the cost function parameterized by $\theta$, $x_t$ in theorem \ref{th:StateMarginalError} and lemma \ref{lemma2} now follows the state distribution of our controller, \ie $p(x_t; \theta)$.

That is,
\begin{align}
     & D_{\sss \textrm{KL}}(p_t(x|x_t;\thetaE)\parallel \frac{1}{K+1}(p(x_t)+p(x_{t+1};\theta)+ \sum_{i=t+2}^{t+K}p(x_i;\thetaE)) \nonumber\\
    = & D_{\sss \textrm{KL}}(\frac{1}{K+1}(p(x_t;\theta)+ p(x_{t+1}; \pi_{re})+ \sum_{i=t+2}^{t+K} p(x;\thetaE))\parallel \frac{1}{K+1}(p(x_t;\theta)+p(x_{t+1};\theta)+ \sum_{i=t+2}^{t+K}p(x_i;\thetaE)) \nonumber\\
     \leq & D_{\sss \textrm{KL}}(\frac{1}{K+1}(p(x_t;\theta)+\sum_{i=t+1}^{t+K} p(x;\thetaE))\parallel \frac{1}{K+1}(p(x_t;\theta)+ \sum_{i=t+1}^{t+K}p(x_i;\theta)) \nonumber\\
    = &  D_{\sss \textrm{KL}}(p_{t}(x|x_t;\thetaE)\parallel p_{t}(x|x_t;\theta))\leq\epsilon \label{equ:subsequence_bound}
\end{align}

Recall that we use $\RHCStateDistribution^t(x; \theta)$ to account for the global state marginal distribution resulted from executing a single optimal control at time step $t$. More specifically, $\RHCStateDistribution^t(x; \theta)$ is defined as the state marginal distribution by executing only the first optimal control from the replanned optimal control sequence at each time step from $i = 0, ..., t$ under $\theta$ and switching to $\thetaE$ thereafter by using the recoverable policy $\pi_{re}$ until $T-1$, \ie $\RHCStateDistribution^t(x;\theta) \approx \frac{1}{T}(\sum_{i=0}^{t+1}p(x_i; \theta) + \sum_{i=t+2}^{T} p(x_i; \thetaE))$. We bound the KL-divergence over global state distribution between two consecutive time step  as follow:
\begin{align}\label{equ:derive_state_bound}
    &D_{\sss \textrm{KL}}(\RHCStateDistribution^{t-1}(x;\theta)\parallel \RHCStateDistribution^t(x;\theta))\nonumber \\
 = &D_{\sss \textrm{KL}}(\frac{1}{T}(\sum_{i=0}^{t}p(x_i; \theta) + p(x_{t+1}; \pi_{re}) + \sum_{i=t+2}^{T} p(x_i; \thetaE)) \parallel \frac{1}{T}(\sum_{i=0}^{t+1}p(x_i; \theta) + p(x_{t+2}; \pi_{re}) + \sum_{i=t+3}^{T}p(x_i; \thetaE)))\nonumber\\
 \approx &  D_{\sss \textrm{KL}}(\frac{1}{T}(\sum_{i=0}^{t}p(x_i; \theta) + \sum_{i=t+1}^{T} p(x_i; \thetaE)) \parallel \frac{1}{T}(\sum_{i=0}^{t+1}p(x_i; \theta) + \sum_{i=t+2}^{T}p(x_i; \thetaE)))\nonumber\\
= & \frac{1}{T} \int (\sum_{i=0}^{t}p(x_i; \theta)+ \sum_{i=t+1}^{T} p(x_i; \thetaE)) \log \frac{\sum_{i=0}^{t}p(x_i; \theta)+ \sum_{i=t+1}^{} p(x_i; \thetaE)}{\sum_{i=0}^{t+1}p(x_i; \theta) + \sum_{i=t+2}^{T}p(x_i; \thetaE)} dx\nonumber\\ 
     \leq &  \frac{1}{T}(\int \sum_{i=0}^{t-1}p(x_i; \theta) \log \frac{\sum_{i=0}^{t-1}p(x_i; \theta) }{\sum_{i=0}^{t-1}p(x_i; \theta)} dx \nonumber \\
    & + \int (p(x_t; \theta)+ \sum_{i=t+1}^{t+K} p(x_i; \thetaE)) \log \frac{  p(x_t; \theta)+ \sum_{i=t+1}^{t+K} p(x_i; \thetaE)}{ p(x_t; \theta)+p(x_{t+1}; \theta) + \sum_{i=t+2}^{t+K}p(x_i; \thetaE)} dx \nonumber \\
     & +  \int \sum_{i=t+K+1}^{T} p(x_i; \thetaE)\log \frac{\sum_{i=t+K+1}^{T} p(x_i; \thetaE)}{\sum_{i=t+K+1}^{T}p(x_i; \thetaE)} dx)\nonumber \\
    = & \frac{1}{T}\int (p(x_t; \theta)+ \sum_{i=t+1}^{t+K} p(x_i; \thetaE)) \log \frac{  p(x_t; \theta) + \sum_{i=t+1}^{t+K} p(x_i; \thetaE)}{ p(x_t; \theta)+p(x_{t+1}; \theta) + \sum_{i=t+2}^{t+K}p(x_i; \thetaE)} dx \nonumber\\
 = & \frac{K+1}{T}\int \frac{1}{K+1} (p(x_t; \theta) + \sum_{i=t+1}^{t+K} p(x_i; \thetaE)) \log \frac{  \frac{1}{K+1} ( p(x_t; \theta) + \sum_{i=t+1}^{t+K} p(x_i; \thetaE))}{  \frac{1}{K+1} (p(x_t; \theta)+p(x_{t+1}; \theta) + \sum_{i=t+2}^{t+K}p(x_i; \thetaE))} dx \nonumber\\
     = & \frac{K+1}{T}D_{\sss \textrm{KL}}(\frac{1}{K+1}(p(x_t; \theta) +\sum_{i=t+1}^{t+K}p(x_i;\thetaE))\parallel \frac{1}{K+1}(p(x_t; \theta) +
     p(x_{t+1}; \theta) + \sum_{i=t+2}^{t+K}p(x_i; \thetaE))) \nonumber\\
     = & \frac{K+1}{T}D_{\sss \textrm{KL}}(p_{t}(x|x_t;\thetaE)\parallel \frac{1}{K+1}(p(x_t; \theta) +
     p(x_{t+1}; \theta) + \sum_{i=t+2}^{t+K}p(x_i; \thetaE))) \nonumber\\
 \leq &  \frac{K+1}{T}\epsilon
\end{align}
The first equality in equation (\ref{equ:derive_state_bound}) follows the definition of $\RHCStateDistribution^t(x;\theta)$ and the second line follows the definition of the recoverable policy $\pi_{re}$. Then, we use lemma \ref{lemma5} to factor out $\frac{1}{T}$ in the third line. The next inequality follows from the generalized log sum inequality stated in lemma \ref{lemma4}, and we have the first and third terms reduce to 0 and are left with the second term in the next line. We apply lemma \ref{lemma5} again to the integral using the constant factor $\frac{1}{K+1}$. In addition, to make the equality hold, we multiply the inverse of the constant factor $K+1$ outside the integral. We observe the integral is now the KL divergence between the expert $p_{t}(x|x_t;\thetaE)$ and one-step-execution of our policy $\frac{1}{K+1}(p(x_t; \theta) + p(x_{t+1}; \theta) + \sum_{i=t+2}^{t+K}p(x_i; \thetaE))$. The final inequality follows from the bound derived in equation \eqref{equ:subsequence_bound}.

For $t=0$, we have $\RHCStateDistribution^0(x;\theta) = \frac{1}{T}(p(\stateinit; \theta) + \sum_{i=1}^{T-1}p(x_i; \thetaE))$. Since the initial state $\stateinit$ for expert and our policy are sampled from the same initial state distribution $\mu$, $p(x_0)$ is independent of $\theta$, \ie $p(\stateinit;\theta) = p(\stateinit;\thetaE) $. Therefore, $\RHCStateDistribution(x;\thetaE) = p(x_0)+ \sum_{i=1}^{T-1}p(x_i; \thetaE) = \RHCStateDistribution^0(x;\theta)$. Moreover, the final global state marginal distribution $\RHCStateDistribution(x; \theta)$ is the same as the $\RHCStateDistribution^{T-1}(x;\theta)$, \ie $\RHCStateDistribution(x; \theta) = \sum_{i=0}^{T-1} p(x_i; \theta) = \RHCStateDistribution^{T-1}(x; \theta)$. For any $t=1,2, .. , T-1$, we have proved $D_{\sss \textrm{KL}}(\RHCStateDistribution^{t-1}(x;\theta)\parallel \RHCStateDistribution^t(x;\theta)) \leq \frac{K+1}{T}\epsilon$.

\end{proof}

Finally, we are prepared to prove theorem \ref{th:StateMarginalError}. 
\begin{proof}
We evaluate the TV distance over the state marginal distribution between the expert policy and our control law. 
\begin{align} \label{equ:TV}
    D_{\sss \textrm{TV}}(\RHCStateDistribution(x; \thetaE)\parallel \RHCStateDistribution(x; \theta)) & =  D_{\sss \textrm{TV}}(\RHCStateDistribution^0(x; \theta)\parallel \RHCStateDistribution^{T-1}(x; \theta))\nonumber\\
    & \leq \sum_{t=1}^{T-1} D_{\sss \textrm{TV}}(\RHCStateDistribution^{t-1}(x; \theta)\parallel \RHCStateDistribution^{t}(x; \theta))
\end{align}
The first equality in equation (\ref{equ:TV}) follows from the fact that $\RHCStateDistribution(x;\thetaE) = \RHCStateDistribution^0(x;\theta)$ and $\RHCStateDistribution(x; \theta) = \RHCStateDistribution^{T-1}(x; \theta)$. We use triangle inequality of the TV distance measures to obtain the inequality in the second line. 

Recall that by Pinsker's inequality, the total variation (TV) distance is related to Kullback–Leibler (KL) divergence by the following inequality: ${D_{\mathrm {TV} } (P\parallel Q)\leq {\sqrt {{\frac {1}{2}}D_{\mathrm {KL} }(P\parallel Q)}}.}$  We apply Pinsker's inequality to each of the TV terms in the second line of equation (\ref{equ:TV}) to bound them by a summation of KL-divergence as shown in the first line of equation (\ref{equ:TV2KL}). Next, given the control sequence distribution for every time step is bounded by $\epsilon$, we apply lemma \ref{lemma2} to show that the resulting state distribution from the optimal control sequences for each time step $t$ is also bounded by $\epsilon$. Next, we use this result and apply Lemma \ref{lemma3} to bound the KL-divergence over the global state marginal distribution between two consecutive time steps by $\frac{K+1}{T}\epsilon$. Second line in equation (\ref{equ:TV2KL}) follows from this result and finally we derive the final bound linear in $T$.
\begin{align} \label{equ:TV2KL}
   D_{\sss \textrm{TV}}(\RHCStateDistribution(x; \thetaE)\parallel \RHCStateDistribution(x; \theta))
    & \leq \sum_{t=1}^{T-1} \sqrt {\frac {1}{2}D_{\sss \textrm{KL}}(\RHCStateDistribution^{t-1}(x; \theta)\parallel \RHCStateDistribution^{t}(x; \theta))} \nonumber\\
    & \leq \sum_{t=1}^{T-1} \sqrt {\frac{(K+1)\epsilon}{2T}}\nonumber\\
    & = (T-1) \sqrt{\frac{K+1}{T}}\sqrt{\epsilon/2}\nonumber  \\
    & \leq T\sqrt{\epsilon/2} 
\end{align}
The last line follows from the fact that $K << T$, so $\sqrt{\frac{K+1}{T}} <1$.
\end{proof}

\section{Extension to Stochastic Dynamics}\label{Appendix:extension_stochastic}

RHIRL optimizes the trajectories in the space of control sequences $p(V)$, whereas $V = \{u_0, u_1, u_2, ..., u_{K-1}\}$ is a sequence of controls. If the system is deterministic, we can apply $V$ to the dynamical system $ x_{t+1} = f(x_t, v_t)$ with start state  $x_0$ and obtain a state sequence $\tau=(x_0, x_1, \ldots, x_{ K})$. We recall that total state trajectory cost of $V$ defined in equation \eqref{equ:statecost} as follows:
\begin{equation}
    S(V, x_0;\theta) = \sum_{k = 0}^{K} g(x_k;\theta) \nonumber
\end{equation}
We use the information-theoretic MPC (MPPI) \citep{MPPI} to solve for an optimal control sequence distribution at the current start state $x_t$ in iteration $t$. The main result of MPPI suggests that, under a deterministic system, the optimal control sequence distribution $Q$ minimizes the ``free energy'' of the dynamical system and  this  free energy can be calculated from the cost of the state trajectory  under $Q$.  Mathematically, the probability density~$p_{Q}(V^*)$, as shown in equation \eqref{equ:optimal_dis} can be 
expressed as a function of the state cost $S(V, x_t;\theta) $,  with respect to a Gaussian ``base'' distribution $B(V_B, \Sigma)$ that depends on the control cost:
\begin{equation}
    p_{ Q}(V^*|U_{\sss B}, \Sigma, x_t;\theta) = {1 \over Z}\,{p_{\sss B}(V^*|U_{\sss B}, \Sigma) \exp(-\frac{1}{\lambda} S(V^*, x_t;\theta))},  \nonumber
\end{equation}
where $Z$ is the partition function. Intuitively, this result shows that the control sequence $V$ that results in lower state-trajectory cost $S(V)$ are exponentially more likely to be chosen.

In this section, we extend RHIRL to stochastic dynamics where $x_{t+1} = f (x_t, v_t, \omega_t)$, where $\omega_t$ is a random variable that
models the independent system noise. More specifically, we assume that $x_{t+1} \sim p(x_{t+1}|x_t, v_t)$. Due to the stochasticity of the dynamics, the state trajectory cost in equation (\ref{equ:statecost}) and the optimal control distribution in equation (\ref{equ:optimal_dis}) are affected. Therefore, we first redefine the trajectory state cost under stochastic dynamics, then derive the counterpart of equation (\ref{equ:optimal_dis}) for the optimal control sequence distribution under the stochastic dynamics, finally we adapt our existing RHIRL algorithm to stochastic dynamics.

\subsection{State Trajectory Cost}

Due to the stochasticity of the dynamics, given the initial state $x_0$, we no longer have a one-to-one mapping from the control sequence $V$ to the resulting state trajectory $\tau = (x_0, x_1, ..., x_K)$. Instead, we have a distribution of state trajectories:
\begin{equation}  \label{equ:traj_dist}
    p(\tau|x_0, V) = \prod_{t=0}^{K-1} p(x_{t+1}|x_t, v_t)
\end{equation}

To accommodate this change, the trajectory state cost of a control sequence $ \tilde{S}(V, x_0;\theta)$ is defined over the distribution of the resulting state trajectories, instead of single trajectory:
 \begin{align} 
    \tilde{S}(V, x_0;\theta) & = \int p(\tau|x_0, V) S(\tau|x_0; \theta)d\tau \\
    & = \int \prod_{t=0}^{K-1} p(x_{t+1}|x_t, v_t)\sum_{t = 0}^{K} g(x_t;\theta) d\tau \label{equ:cost_sto}
\end{align}

We always measure the preferences over the control sequence $V$ by their resulting state trajectories $\tau$. Hence, when the resulting trajectories changes from a single deterministic sequence of states to a distribution of state trajectories, we adapt our measure of the resulting cost: under the deterministic dynamics where each $V$ uniquely maps to the same state trajectory $\tau$, the state trajectory cost is the cost of that specific trajectory; while under the stochastic dynamics where the same control sequence $V$ maps to a  distribution of $\tau$, the state trajectory cost of a control sequence is now defined as the expected state cost of the distribution of trajectory. We measure the state trajectory cost of a control sequence $S(V, x_0; \theta)$, instead of simply a state trajectory cost on the states itself $S(\tau, x_0;\theta)$, because we want to use this measure to directly optimize the control sequence.

\subsection{Optimal Control Sequence Distribution}

Next, we derive the optimal control sequence distribution under the stochastic dynamics. Our derivation is based on MPPI \citep{MPPI}, which uses the ``free-energy'' principle to derive the optimal control sequence distribution under deterministic dynamics. 

\begin{definition}[Free-energy (\citep{relative_energy}, Definition 1]
Let $\mathbb{P} \in \mathcal{P}(\mathcal{Z})$ and the function $\mathcal{J}(x)$:
$\mathcal{Z} \to \mathbb{R}$ be a measurable function. The the term:
\begin{equation}
    \mathbb{E}(\mathcal{J}(x)) = \log \int \exp(\rho \mathcal{J}(x)) d\mathbb{P}
\end{equation}
is called free energy of $\mathcal{J}(x)$ with respect to $\mathbb{P}$, $\rho$ is a constant.
\end{definition}

Now we have the free-energy of a control system under stochastic dynamics as stated below. It has a ``Gaussian'' base control sequence distribution $B(U_{\sss B}, \Sigma)$ such that its control sequence distribution follows $p_{\sss B}(V^*|U_{\sss B}, \Sigma)$ whereas $\Sigma$ is the Gaussian control noise covariance matrix. $\tilde{S}(V; \theta)$ denotes the state trajectory cost function.

\begin{equation}
    \mathcal{F}(\tilde{S}, p_{\sss B}, x_0, \lambda;\theta) = \log (\mathbb{E}_{p_{\sss B}}[\exp (-\frac{1}{\lambda}\tilde{S}(V, x_0; \theta))]),
\end{equation}
$\lambda \in \mathbb{R}^{+}$ is the inverse temperature of the control system.

Suppose now we have another control sequence distribution with probability measure $p(V)$ and these two distributions are absolutely continuous, then we can rewrite the free-energy w.r.t  $p_{\sss B}(V)$ using the expectation over the density of $p(V)$ use the standard importance sampling trick:
\begin{align}
      \mathcal{F}(\tilde{S}, p_{\sss B}, x_0, \lambda; \theta) &= \log (\mathbb{E}_{p_{\sss B}}[\exp (-\frac{1}{\lambda}\tilde{S}(V, x_0; \theta))]) \\
      &= \log (\mathbb{E}_{p_{\sss B}}[\exp (-\frac{1}{\lambda}\tilde{S}(V, x_0; \theta)\frac{p_{\sss B}(V^*|U_{\sss B}, \Sigma)}{p(V)})]) \\
      & \ge  \mathbb{E}_{p}[\log(\exp (-\frac{1}{\lambda}\tilde{S}(V, x_0; \theta)\frac{p_{\sss B}(V^*|U_{\sss B}, \Sigma)}{p(V)}))] \\
      & = -\frac{1}{\lambda}\mathbb{E}_{p}[\tilde{S}(V, x_0; \theta)+\lambda \log(\frac{p(V)}{p_{\sss B}(V^*|U_{\sss B}, \Sigma)})] \\
      & = -\frac{1}{\lambda}(\mathbb{E}_{p}[\tilde{S}(V, x_0; \theta)]+\lambda \mathbb{E}_{p}[\log(\frac{p(V)}{p_{\sss B}(V^*|U_{\sss B}, \Sigma)})]) \\
      & = -\frac{1}{\lambda}(\mathbb{E}_{p}[\tilde{S}(V, x_0; \theta)]+\lambda D_{\sss \textrm{KL}}(p(V)|| p_{\sss B}(V^*|U_{\sss B}, \Sigma)))
\end{align}

Therefore, we have 

\begin{equation} \label{equ:bound}
 \mathcal{F}(S, p_{\sss B}, x_0, \lambda;\theta) \ge  -\frac{1}{\lambda}(\mathbb{E}_{p}[\tilde{S}(V, x_0; \theta)]+\lambda D_{\sss \textrm{KL}}(p(V)|| p_{\sss B}(V^*|U_{\sss B}, \Sigma)))
\end{equation}
The right-hand side is the lower bound of the free-energy of the control system. We use $p_Q(V)$ to denote the optimal control sequence distribution. This distribution is only optimal if and only if the bound in the equation above is tight, i.e. $ \mathcal{F}(\tilde{S}, p_{\sss B}, x_0, \lambda;\theta) = -\frac{1}{\lambda} \mathbb{E}_{p_Q}[\tilde{S}(V, x_0; \theta)]+\lambda D_{\sss \textrm{KL}}(p_Q(V)|| p_{\sss B}(V^*|U_{\sss B}, \Sigma))$. 

We claim that the optimal control sequence distribution $\tilde{p_Q}$ under the stochastic dynamics is as follows:
\begin{equation} \label{equ:stochastic_obj}
    \tilde{p_Q}(V^*|U_{\sss B}, \Sigma, x_0; \theta) = {1 \over Z}\,{ \exp(-\frac{1}{\lambda} \tilde{S}(V, x_0; \theta))p_{\sss B}(V^*|U_{\sss B}, \Sigma)}, 
\end{equation}
whereas $Z = \int \exp(-\frac{1}{\lambda} \tilde{S}(V, x_0;\theta))p_{\sss B}(V^*|U_{\sss B}, \Sigma) dV$ is the partition function.

We prove that equation (\ref{equ:stochastic_obj}) is the optimal control sequence distribution under stochastic dynamics by showing that this $\tilde{p_Q}(V^*|U_{\sss B}, \Sigma, x_0; \theta)$ tightens the bound of free-energy in equation (\ref{equ:bound}). We substitute $\tilde{p_Q}(V^*|U_{\sss B}, \Sigma, x_0; \theta)$ into the RHS of equation (\ref{equ:bound}) and simplify the expression of the KL-divergence:

\begin{align}
    \mathcal{F}(\tilde{S}, p_{\sss B}, x_0, \lambda; \theta) & \ge  -\frac{1}{\lambda}(\mathbb{E}_{\tilde{p_Q}}[\tilde{S}(V, x_0; \theta)]+\lambda D_{\sss \textrm{KL}}(\tilde{p_Q}(V^*|U_{\sss B}, \Sigma, x_0; \theta)|| p_{\sss B}(V^*|U_{\sss B}, \Sigma)))\\
    & = -\frac{1}{\lambda}(\mathbb{E}_{\tilde{p_Q}}[\tilde{S}(V, x_0; \theta)]+\lambda \mathbb{E}_{\tilde{p_Q}}[\log(\frac{\tilde{p_Q}(V^*|U_{\sss B}, \Sigma, x_0; \theta)}{p_{\sss B}(V^*|U_{\sss B}, \Sigma)})]) \\
     & = -\frac{1}{\lambda}\mathbb{E}_{\tilde{p_Q}}[\tilde{S}(V, x_0; \theta)] - \mathbb{E}_{\tilde{p_Q}}[\log(\frac{ {1 \over Z}\,{ \exp(-\frac{1}{\lambda} \tilde{S}(V, x_0; \theta))p_{\sss B}(V^*|U_{\sss B}, \Sigma)}}{p_{\sss B}(V^*|U_{\sss B}, \Sigma)})] \\
      & = -\frac{1}{\lambda}\mathbb{E}_{\tilde{p_Q}}[\tilde{S}(V, x_0; \theta)] - ({1 \over \lambda}\mathbb{E}_{\tilde{p_Q}}[\tilde{S}(V, x_0; \theta)] - \log(Z)) 
\end{align}

Next we substitute the expression for the partition function $Z$ and we found that the RHS is exactly the definition of the the free-energy of the control system with base distribution $B(U_{\sss B}, \Sigma)$:
\begin{align}
      \mathcal{F}(\tilde{S}, p_{\sss B}, x_0, \lambda; \theta)  & \ge \log(Z) \\
      & = \log (\int \exp(-\frac{1}{\lambda} \tilde{S}(V, x_0; \theta))p_{\sss B}(V^*|U_{\sss B}, \Sigma) dV)  \\
      & = \log  (\mathbb{E}_{p_{\sss B}}[\exp (-\frac{1}{\lambda}\tilde{S}(V, x_0; \theta))]) \\
      & =   \mathcal{F}(\tilde{S}, p_{\sss B}, x_0, \lambda; \theta) 
\end{align}
The final equality forces the inequality to be tight. Therefore, $p_Q(V^*|U_{\sss B}, \Sigma, x_0; \theta)$ in equation (\ref{equ:stochastic_obj}) is the optimal control sequence distribution under the stochastic dynamics.

We observe that the optimal control sequence distribution under deterministic system $p_Q(V^*|U_{\sss B}, \Sigma, x_0; \theta)$ in equation (\ref{equ:optimal_dis}) and that under the stochastic dynamics $ \tilde{p_Q}(V^*|U_{\sss B}, \Sigma, x_0; \theta)$ in equation (\ref{equ:stochastic_obj}) only differs in the calculation of the state trajectory cost of the control sequences. Intuitively, it means that under deterministic dynamics, we choose the control sequence $V$ that will, for sure, leads to a state trajectory with lower cost; while when we extend to stochastic dynamics, the control sequence $V$ that results in lower state-trajectory cost $S(V)$ in expectation are exponentially more likely to be chosen. Practically, now we need more samples for a single control sequence to compute the expectation in equation (\ref{equ:cost_sto}).

\subsection{RHIRL under Stochastic Dynamics}

Next, we adapt our RHIRL algorithm to this new state trajectory cost measure $\tilde{S}(V, x_0;\theta)$ and the optimal control sequences $\tilde{p_{\sss Q}}(V^*|U_{\sss B}, \Sigma, x_0; \theta)$ under stochastic dynamics. We recall that under the deterministic dynamics, \algo uses importance sampling in equation \eqref{equ:grad} to estimate the $\frac{\partial \mathcal{L}}{\partial \theta}$ so as to update the cost function parameter $\theta$:
\begin{equation*}
    \frac{\partial }{\partial \theta}\mathcal{L}(\theta; D, x_0) \approx \frac{1}{N}\sum_{i=1}^{N}\frac{1}{\lambda}\frac{\partial}{\partial \theta}S(V_i, x_0;\theta) -\frac{1}{M}\sum_{j=1}^{M}\frac{1}{\lambda}w(V_j)\frac{\partial }{\partial \theta}S(V_j, x_0;\theta),
\end{equation*}
whereas the $N$ control sequences in the first term are from the expert demonstration $D_t$ and the $M$ control sequences in the second term are from our approximated optimal control sequence distribution $p_{\sss Q}(V^*)$, and $w(V_j)$ is the importance sampling weight. 

To estimate $\frac{\partial \mathcal{L}}{\partial \theta}$, we need to calculate/approximate the importance sampling weight $w(V)$, and the derivative of the state trajectory cost $\frac{\partial}{\partial \theta}S(V, x_0; \theta)$ w.r.t $\theta$. We recall that the importance sampling weight $w(V)$ depends on the state trajectory cost $S(V, x_0; \theta)$ in equation \eqref{eqn:weights}:
\begin{equation*}
    w(V) \propto \exp(-\frac{1}{\lambda}\bigg( S(V, x_0;\theta)+\lambda\sum_{k=0}^{K-1}u_k^{\intercal}\Sigma^{-1}v_k)\bigg)
\end{equation*}

Under the deterministic dynamics, the importance sampling weight $w(V)$ is estimated using Monte-Carlo approximation with $M$ state trajectory samples as follows:
\begin{align}
    w(V) & \approx \frac{ \exp \bigg(-\frac{1}{\lambda}(S(V, x_0; \theta) + \lambda \sum_{k=0}^{K-1}u_k^{\intercal}\Sigma^{-1}v_k)\bigg)}{ \sum_{j=0}^{M-1} \exp \bigg(-\frac{1}{\lambda}(S(V_j, x_0; \theta) + \lambda \sum_{k=0}^{K-1}u_k^{j\intercal}\Sigma^{-1}v^j_k)\bigg)}\\
    &= \frac{\exp \bigg(-\frac{1}{\lambda}(\sum_{t=0}^{K} g(x_t; \theta) + \lambda \sum_{k=0}^{K-1}u_k^{\intercal}\Sigma^{-1}v_k)\bigg)}{ \sum_{j=0}^{M-1} \exp \bigg(-\frac{1}{\lambda}(\sum_{t=0}^{K} g(x^j_t; \theta) + \lambda \sum_{k=0}^{K-1}u_k^{j\intercal}\Sigma^{-1}v^j_k)\bigg)}
\end{align}

Since the state trajectory cost $S(V, x_0; \theta)$ is a linear sum of the cost of all states, $\frac{\partial S}{\partial \theta}$ can be directly computed as follows:
\begin{equation}
    \frac{\partial}{\partial \theta}S(V, x_0; \theta) = \frac{\partial}{\partial \theta} \sum_{t=0}^K g(x_t;\theta) = \sum_{t=0}^K \frac{\partial}{\partial \theta}g(x_t;\theta)
\end{equation}

To extend RHIRL to stochastic dynamics, when the state trajectory cost function is now $\tilde{S}(V, x_0;\theta)$, we need to redefine how to estimate $\tilde{w}(V)$ and consequently $\frac{\partial \tilde{S}}{\partial \theta}$.

When extend to stochastic dynamic, we have the following:
\begin{equation}
    \tilde{w}(V) \propto \exp(-\frac{1}{\lambda}\bigg(\tilde{S}(V, x_0; \theta)+\lambda\sum_{k=0}^{K-1}u_k^{\intercal}\Sigma^{-1}v_k)\bigg)
\end{equation}
 and we still adopts Monte-carlo sampling to approximate $\tilde{w}(V)$. However, since $\tilde{S}(V, x_0;\theta)$ now measures the expected state trajectory cost under the stochastic dynamics, we need to go one step further and use sampling to estimate $\tilde{S}(V, x_0;\theta)$ using $M^s$ number of state trajectories $\tau_h = (x_0, x_1^h, ... , x_K^h)$ per $(x_0, V)$ pair:
\begin{equation} \label{equ:sample_cost}
    \tilde{S}(V, x_0;\theta) \approx \frac{1}{M^s}\sum_{h = 0}^{M^s-1}\sum_{t = 0}^{K} g(x_t^h;\theta) 
\end{equation}

Therefore, now the importance sampling weight $\tilde{w}(V)$ is approximated from $M \times M^s$ state trajectories, with $M^s$ trajectories from each $(x_0, V_j)$ pair as follows:
\begin{align}
    \tilde{w}(V) & \approx \frac{ \exp \bigg(-\frac{1}{\lambda}(\tilde{S}(V, x_0;\theta) + \lambda \sum_{k=0}^{K-1}u_k^{\intercal}\Sigma^{-1}v_k)\bigg)}{ \sum_{j=0}^{M-1} \exp \bigg(-\frac{1}{\lambda}(\tilde{S}(V_j, x_0;\theta) + \lambda \sum_{k=0}^{K-1}u_k^{j\intercal}\Sigma^{-1}v^j_k)\bigg)}\\
    & \approx \frac{ \exp \bigg(-\frac{1}{\lambda}(\frac{1}{M^s}\sum_{h = 0}^{M^s-1}\sum_{t = 0}^{K} g(x_t^h;\theta)  + \lambda \sum_{k=0}^{K-1}u_k^{\intercal}\Sigma^{-1}v_k)\bigg)}{ \sum_{j=0}^{M-1} \exp \bigg(-\frac{1}{\lambda}(\frac{1}{M^s}\sum_{h = 0}^{M^s-1}\sum_{t = 0}^{K} g(x_t^{j^h};\theta)  + \lambda \sum_{k=0}^{K-1}u_k^{j\intercal}\Sigma^{-1}v^j_k)\bigg)}
\end{align}
We emphasize that under the stochastic dynamics, we use the state trajectory samples to estimate both the state trajectory cost $\tilde{S}(V, x_0;\theta)$ and the importance sampling weight $\tilde{w}(V)$. Since the $\tilde{S}$ now measures the expected cost over a distribution of trajectories, we need more samples to estimate $\tilde{w}(V)$ compared to the deterministic setting.

Moreover, in the final $\frac{\partial \mathcal{L}}{\partial \theta}$, we need to differentiate $\tilde{S}(V, x_0;\theta)$ w.r.t. $\theta$. Since $\tilde{S}$ is estimated from sampling, we have:
\begin{equation} \label{equ:sample_cost_sto}
    \frac{\partial}{\partial \theta}\tilde{S}(V, x_0; \theta) \approx \frac{\partial}{\partial \theta} \frac{1}{M^s}\sum_{h = 0}^{M^s-1}\sum_{t = 0}^{K} g(x_t^h;\theta) \approx  \frac{1}{M^s}\sum_{h = 0}^{M^s-1}\sum_{t = 0}^{K}\frac{\partial}{\partial \theta}g(x^h_t;\theta)
\end{equation}

Finally, we summarize how to extend RHIRL to stochastic dynamics. In stochastic dynamics, each control sequence $V$ will map to a distribution of state trajectory $p(\tau|V, x_0)$. Hence, we adapt our measure of state trajectory cost $S(V, x_0;\theta)$ from a single trajectory to be the expected cost over a distribution of state trajectories $\tilde{S}(V, x_0;\theta)$ in equation (\ref{equ:cost_sto}). Next, we revise the optimal control sequence distribution to a stochastic setting $ \tilde{p_Q}(V^*|U_{\sss B}, \Sigma, x_0; \theta)$ in equation (\ref{equ:stochastic_obj}). More specifically, we show under the stochastic dynamics, the optimal control sequences $V^*$ is chosen based on the expected cost of its resulting state trajectories. Finally, under this new state trajectory cost $\tilde{S}(V, x_0;\theta)$ and the optimal control sequence distribution  $ \tilde{p_Q}(V^*|U_{\sss B}, \Sigma, x_0; \theta)$, we adapt the approximation of the importance sampling weight $\tilde{w}(V)$ and consequently the gradient of the overall loss w.r.t $\theta$ by adding one more sampling process to estimate the new state trajectory cost. Moreover, in practice, we can use the same set of samples to estimate both state trajectory cost $\tilde{S}(V, x_0;\theta)$ and the importance sampling weights $\tilde{w}(V)$.

\section{Experimental Details}\label{Appendix:exp}
In this section, we list down the implementation details of 
\algo and the baselines. The code is included in the supplementary material. We also report the hyperparameters used in the experiments, the detailed network architectures, training procedures and evaluation procedures used for our experiments. 

\subsection{Practical Issues of \algo}

\textbf{Control Noise Covariance Approximation} \label{Appendix:sigma_approximation}

The actual control noise covariance $\noiseMatrix$ is unknown to \algo and the baselines. However, \algo uses the noise covariance matrix $\noiseMatrix$ to sample the controls around the nominal control (Algorithm 1, line 6) and calculate the quadratic control cost in Equation~\eqref{equ:cost_sto}. Since we have no access to the true $\noiseMatrix$, \algo approximates $\noiseMatrix$ as a constant factor of the identity matrix $\beta I$, whereas $\beta$ is the hyperparameter we optimize using grid search and $I$ is the identify matrix with its width equals to the dimension of the action space. Therefore, instead of sampling the controls from  $ \mathcal{N}(V | U,\Sigma)$, we sample from  $ \mathcal{N}(V | U,\beta I)$ in practice. We also use $\beta I$ in Equation~\eqref{equ:cost_sto} to replace the unknown $\noiseMatrix$. Even in the noise-free environment, we set $\beta$ to a non-zero value to foster exploration; otherwise, the importance sampling degenerates to the single nominal control. 

Our experiment shows that \algo is robust to the choice of $\beta$: the cost learning performs well even if $\beta I \neq \noiseMatrix$. This may be attributed to the fact that we jointly optimize the state cost function and $\beta$. Therefore the learned state cost function may compensate for the inaccurate approximation for $\noiseMatrix$.

\textbf{Numerical Stability} 

Equation\eqref{eqn:weights} forms the basis of importance sampling and estimation. However, the learned cost can be a huge negative number, which causes numerical instability in estimating the importance weights. To mitigate this issue, we subtract the minimum trajectory cost $S_{min}$ from all rollouts to improve the numeric stability. Since subtracting the same number from all rollouts does not change the order of the preference, this operation does not affect the optimality of our derivation.

\textbf{Nominal Control Initialization and Local Optimality}

\algo samples around the nominal control sequence to collect the samples for importance sampling. However, if the initial nominal control sequence performs poorly, it is not easy to generate any good samples to improve the current control sequences. To mitigate this problem, we add an exploration strategy to the sampling process in (Algorithm 1, line 6): with probability $\alpha$, we continue with the standard sampling strategy to sample around the nominal control; with probability $1-\alpha$ we sample uniformly from the entire action space. This helps \algo to correct from the unsuitable nominal control initialization and also helps \algo to escape the local optimal solution. We set $\alpha = 0.5$ for all tasks in our experiments.  

\textbf{Control Smoothness}

Updating the optimal control by importance sampling might cause some jerk in the control space. In order to make the control change smoothly in its local space, we apply a Savitzky–Golay filter over the time horizon dimension to constrain the control that does not change too much over the time horizon.  

\subsection{Training}
\label{Appendix:task}
We list the hyper-parameters of \algo for different tasks. These hyper-parameters were selected via grid search. 

\begin{table}[ht]
\centering

\begin{tabular}{llllllll}\toprule
Task   & K & $\beta$  & batch size   & $\lambda$ & lr & weight decay \\
                      \bottomrule
Pendulum-v0                & 20  & 0.8  & 50  & 0.10 & 1e-4 & 8e-5\\
LunarLanderContinuous-v2   & 40  & 0.6  & 200 & 0.10 & 1e-4 & 8e-5\\
Hopper-v2                  & 20  & 0.8  & 100 & 0.10& 1e-4 & 8e-5 \\
Walker2d-v2.               & 30  & 0.6  & 150 & 0.10 & 1e-4 & 8e-5\\
Ant-v2                     & 15  & 1.2  & 200 & 0.10 & 1e-4 & 8e-5  \\
CarRacing-v0               & 15  & 1.0  & 200 & 0.10 & 1e-4 & 8e-5\\
\bottomrule
\end{tabular}
\end{table}

The implementation of the baselines (f-IRL, AIRL and GAIL) are adapted from f-IRL's \citep{fIRL} official repository. We use the hyperparameters reported in f-IRL for the MuJoCo tasks and performed grid search on the hyperparameters for the rest of the tasks. SAC\citep{SAC} is used as the base MaxEnt RL algorithm for both expert policy and the baselines optimization algorithm. We use a tanh squashed Gaussian as the policy network for Pendulum-v0, LunarLander-v2, and the MuJoCo tasks; and we use a Gaussian Convolutional policy as the policy network for CarRacing-v0. The mean and std of the Gaussian are parameterized by a ReLU MLP of size (64, 64). Adam is used as the optimizer. We use the reported SAC temperature,$\alpha = 0.2$, reward scale $c = 0.2$, and gradient penalty coefficient $\lambda = 4.0$. The rest of the hyperparameters for f-IRL, GAIL and AIRL are listed below. 
\begin{table}[ht]
\fontsize{9}{9}\selectfont
\centering
\begin{tabular}{lllllll}\toprule
Task & SAC learning rate & SAC replay buffer size & \makecell{Reward/Value\\ model learning rate} & l2 weight decay \\
\bottomrule
Pendulum-v0              & 1e-4 & 100000 & 1e-5 & 1e-3 \\
LunarlanderContinuous-v2 & 1e-3  & 100000 & 1e-5  & 1e-3\\
Hopper-v2            & 1e-5 & 1000000 & 1e-5  & 1e-3 \\
Walker2d-v2         & 1e-5  & 1000000 & 1e-5  & 1e-3 \\
Ant-v2         & 3e-4 & 1000000  & 1e-4   & 1e-3\\
CarRacing-v0             & 4e-4 & 10000000 & 1e-4  & 1e-3\\
\bottomrule
\end{tabular}
\end{table}

\subsection{Reward Function and Discriminator Network Architectures}
We use the same neural network architecture to parameterize the cost-function/reward-function/discriminator for all methods. For continuous control task with raw state input, i.e. pendulum, lunarlander and the MuJoCo tasks, we use two-layer of MLP with ReLU activation function to parameterized the cost function/discriminator. The hidden size for Pendulum-v0 is (32, 32), and (64, 64) for the rest of the tasks.

For continuous control task with image input, i.e. carracing, we use a four convolutional layer with kernel size $3 \times 3$ as the feature extractor. The output of the CNN layer is vector with size (128,) and is fed into the same reward network as describe above.

\subsection{Additional Experiments Results}
\label{sec:AddResults}

We report the average returns and the standard deviation for \tabref{tab:scalablity} and \tabref{tab:noise} in \tabref{tab:scalability_full_table} and \tabref{tab:noise_full_table} respectively. The mean and standard deviation computed from 3 trials for each entry of the tables.

\begin{table}[htp]
\fontsize{8}{8}\selectfont
\caption{Performance of \algo, f-IRL, GAIL,  and AIRL. We report the mean and the standard deviation of the policy returns using the ground-truth task reward. Higher values indicate better performance. }
\vspace{10px}
\centering
    \begin{tabular}{l@{\hspace*{6pt}}l@{\hspace*{3pt}}ccc}\toprule
        &&\textbf{\makecell{No Noise\\ $\Sigma = 0$}} & \textbf{\makecell{Mild Noise\\$\Sigma = 0.2$}} & \textbf{\makecell{High Noise\\$\Sigma = 0.5$}}\\\midrule
        Pendulum  
            & Expert & -154.69 $\pm$ 50.05 & -156.50 $\pm$ 70.72 & -168.54 $\pm$ 80.89\\
            & \algo & -125.95 $\pm$ 1.21 & -122.33 $\pm$ 3.44 & -132.39 $\pm$ 10.36 \\
             & f-IRL & -121.94 $\pm$ 97.21 & -127.51 $\pm$ 104.55 & -197.36 $\pm$ 106.92\\
             & AIRL & -131.64 $\pm$ 1.16 & -184.62 $\pm$ 88.16 & -203.12 $\pm$ 80.57\\
             & GAIL & -207.05 $\pm$ 57.41 & -207.14 $\pm$ 57.52 & -253.85 $\pm$ 181.84\\
        \midrule
        LunarLander
             & Expert & 235.13 $\pm$ 43.59 & 222.65 $\pm$ 56.35 & 164.52 $\pm$ 36.79\\
             & \algo & 246.39 $\pm$ 10.96 & 233.73 $\pm$ 23.75 & 198.23 $\pm$ 47.8\\
             & f-IRL & 179.03 $\pm$ 9.19 & 141.73 $\pm$ 11.81 & 121.67 $\pm$ 22.77\\
             & AIRL & 174.49 $\pm$ 35.17 & 132.76 $\pm$ 85.59 & 95.61 $\pm$ 19.25\\
             & GAIL & 169.98 $\pm$ 15.43 & 125.5 $\pm$ 16.78 & 100.24 $\pm$ 79.04 \\
        \midrule
        Hopper
            & Expert & 3222.48 $\pm$ 390.65 & 3159.32 $\pm$ 520.00 & 2887.72 $\pm$ 483.93\\
            & \algo & 3071.63 $\pm$ 122.03 & 3121.72 $\pm$ 278.98 & 2776.2 $\pm$ 345.90\\
             & f-IRL & 3080.34 $\pm$ 458.96  & 2580.19 $\pm$ 637.21 & 1270.24 $\pm$ 539.84\\
             & AIRL & 18.9 $\pm$ 0.79 & 33.52 $\pm$ 3.86 & 18.38 $\pm$ 7.84\\
             & GAIL & 2642.59 $\pm$ 187.33 & 1576.25 $\pm$ 1051.98 & 702.33 $\pm$ 151.37\\
             
         \midrule
         Walker2d
            & Expert & 4999.47 $\pm$ 55.99 & 4500.43 $\pm$ 114.48& 3624.48 $\pm$ 95.05\\
            & \algo & 4939.44 $\pm$ 100.28 &4473.332 $\pm$ 324.34 & 3446.55 $\pm$ 507.89\\
             & f-IRL & 4927.92 $\pm$ 529.95 & 3697.36 $\pm$ 711.56 & 2831.91 $\pm$ 993.76\\
             & AIRL & -2.51 $\pm$ 0.69 & 22.24 $\pm$ 10.74 & 6.5 $\pm$ 5.03\\
             & GAIL & 2489.04 $\pm$ 813.31 & 2884.35 $\pm$ 59.88 & 1840.62 $\pm$ 778.3\\
             
        \midrule
        Ant 
            & Expert & 5759.22 $\pm$ 173.57 &  2557.37 $\pm$ 501.95 & 252.62 $\pm$ 91.44\\
            & \algo & 4987.67 $\pm$ 149.2 & 2373.32 $\pm$ 529.3 & 230.8 $\pm$ 253.39\\
             & f-IRL & 5022.42 $\pm$ 108.07 & 2034.87 $\pm$ 262.29 & 197.2 $\pm$ 200.45\\
             & AIRL & 1000.4 $\pm$ 0.79 & 849.05 $\pm$ 30.15 & -7.43 $\pm$ 6.01\\
             & GAIL & 2784.87 $\pm$ 301.66 & 1022.04 $\pm$ 580.49 & -416.69 $\pm$ 292.23\\
             
        \midrule
        CarRacing 
        & Expert & 903.25 $\pm$ 0.23 & 702.01 $\pm$ 0.3 & 281.12 $\pm$ 0.34\\
        & \algo & 359.61 $\pm$ 40.32 & 206.21 $\pm$ 19.87  & 53.97 $\pm$ 3.24\\
             & f-IRL & 85.45 $\pm$ 47.4 & 18.32 $\pm$ 27.89 & 2.04 $\pm$ 13.8\\
             & AIRL & -21.97 $\pm$ 2.67 & -25.25 $\pm$ 5.98 & -32.31 $\pm$ 7.43\\
             & GAIL & 2.62 $\pm$ 3.41 & -7.65 $\pm$ 4.77 & -15.88 $\pm$ 5.89\\
             
        \bottomrule
    \end{tabular}\label{tab:scalability_full_table}
\end{table}

\begin{table}[htp]
\fontsize{8}{8}\selectfont
\caption{Generalization of learned cost functions over different noise levels.}
\vspace*{10pt}
\centering
\begin{tabular}{lcccc}\toprule
 & \multirow{2}{*}{}{}& \textbf{\makecell{Noise-free for learning }} & \multicolumn{2}{l}{\textbf{\makecell{Noise Level $\Sigma$ for Testing }}} \\
 && & \textrm{0.2} & \textrm{0.5}    \\
                          \midrule

Pendulum 
            & \algo  & -125.95 $\pm$ 1.21 &-125.01 $\pm$ 4.53 & -126.4 $\pm$ 7.73\\ 
            & f-IRL   & -121.94 $\pm$ 97.21 & -199.44 $\pm$ 96.99 & -220.74$\pm$79.75 \\
            & AIRL  & -131.64 $\pm$ 1.16 &  -247.86$\pm$11.44 &-304.48$\pm$20.78 \\
            & GAIL & -207.05 $\pm$ 57.41 & -220.6$\pm$ 69.82 &  -270.81$\pm$ 79.68 \\ 
                          \midrule
LunarLander 
            & \algo & 246.39 $\pm$ 10.96 & 205.66$\pm$24.67 &175.82$\pm$52.12\\
            & f-IRL  & 179.03 $\pm$ 9.19 & 121.80$\pm$20.94  & 102.06$\pm$22.31\\
            & AIRL  & 174.49 $\pm$ 35.17 & 31.46$\pm$9.68 &22.29$\pm$14.01\\
             & GAIL  & 169.98 $\pm$ 15.43 & 101.80 $\pm$ 23.12 & 78.33$\pm$ 24.15\\
                          \midrule
Hopper 
            & \algo  & 3071.63 $\pm$ 122.03 & 2577.28 $\pm$409.33 &2152.08 $\pm$342.21\\
            & f-IRL  & 3080.34 $\pm$ 458.96 & 2110.52$\pm$26.71 & 1984.29$\pm$31.88\\
            & AIRL   & 18.9 $\pm$ 0.79 & 18.86$\pm$4.80 & 8.78$\pm$10.89 \\
            & GAIL   & 2642.59 $\pm$ 187.33 & 215.29$\pm$ 27.76  & 132.15$\pm$ 30.20\\
                          \midrule
Walker2d
            & \algo   & 4939.44 $\pm$ 100.28 & 4039.44$\pm$39.2 & 3440.23$\pm$531.08\\
            & f-IRL   & 4927.92 $\pm$ 529.95 & 2976.66$\pm$396.57  & 1090.11$\pm$1389.56\\
            & AIRL  & -2.51 $\pm$ 0.69 & 1380.84$\pm$364.95 & 1787.15$\pm$230.94 \\
            & GAIL  & 2489.04 $\pm$ 813.31 & 103.15$\pm$ 121.84  & 124.15$\pm$82.84 \\
                          \midrule
Ant
            & \algo  & 4987.67 $\pm$ 149.2 & 3192.82$\pm$162.12	& 867.08$\pm$204.28\\
            & f-IRL  & 5022.42 $\pm$ 108.07 & 2042.41$\pm$129.89 & 472.77$\pm$110.2\\
            & AIRL   & 1000.4 $\pm$ 0.79 & 845.69 $\pm$29.01 & 0.69 $\pm$20.49 \\
            & GAIL   & 2784.87 $\pm$ 301.66 & -6.41$\pm$ 21.17 & -79.89$\pm$ 142.43\\
             \midrule
CarRacing
            & \algo & 359.61 $\pm$ 40.32 & 261.78$\pm$54.44	& 110.12$\pm$58.90\\
            & f-IRL & 85.45 $\pm$ 47.4 & 16.12$\pm$67.82 & -24.78$\pm$2.12\\
            & AIRL  & -21.97 $\pm$ 2.67 & -27.09$\pm$6.65 & -23.96$\pm$4.11  \\
            & GAIL  & 2.62 $\pm$ 3.41 & -6.41$\pm$3.22 & -49.89$\pm$7.98 \\
\bottomrule
\end{tabular}\label{tab:noise_full_table}
\end{table}
\end{document}